\documentclass{article}

\usepackage[preprint]{neurips_2026}

\usepackage[utf8]{inputenc} 
\usepackage[T1]{fontenc}    
\usepackage{hyperref}       
\usepackage{url}            
\usepackage{booktabs}       
\usepackage{amsfonts}       
\usepackage{nicefrac}       
\usepackage{microtype}      
\usepackage{xcolor}         

\usepackage{algpseudocode}
\usepackage{algorithm}
\usepackage{algorithmicx}
\usepackage{multirow}
\usepackage{graphicx}
\usepackage{amsmath}
\usepackage{amssymb}
\usepackage{amsthm}
\usepackage{amsfonts}
\usepackage{wrapfig}

\newtheorem{theorem}{Theorem}[section]
\newtheorem{lemma}[theorem]{Lemma}
\newtheorem{proposition}[theorem]{Proposition}

\title{Safety Alignment as Continual Learning: Mitigating the Alignment Tax via Orthogonal Gradient Projection}

\author{
  Guanglong Sun \thanks{First author.}\\
  School of Life Sciences, IDG/McGovern Institute for Brain Research \\
  Tsinghua University, Beijing, China \\
  \texttt{sgl23@mails.tsinghua.edu.cn} \\
  \AND
  Siyuan Zhang \thanks{First author.} \\
  Dept. of Comp. Sci. and Tech., Institute for AI, Tsinghua-Bosch Joint ML Center, THBI Lab, BNRist Center \\
  Tsinghua University, Beijing, China \\
  \texttt{zhang-sy24@mails.tsinghua.edu.cn} \\
  \And
  Liyuan Wang \\
  Department of Psychological and Cognitive Sciences \\ 
  Tsinghua University, Beijing, China \\
  \texttt{liyuanwang@tsinghua.edu.cn} \\
  \And
  Jun Zhu \\
  Dept. of Comp. Sci. and Tech., Institute for AI, Tsinghua-Bosch Joint ML Center, THBI Lab, BNRist Center \\
  Tsinghua University, Beijing, China \\
  \texttt{dcszj@tsinghua.edu.cn} \\
  \And
  Hang Su \thanks{Corresponding author.} \\
  Dept. of Comp. Sci. and Tech., Institute for AI, Tsinghua-Bosch Joint ML Center, THBI Lab, BNRist Center \\
  Tsinghua University, Beijing, China \\
  \texttt{suhangss@mail.tsinghua.edu.cn} \\
  \And
  Yi Zhong \thanks{Corresponding author.}\\
  School of Life Sciences, IDG/McGovern Institute for Brain Research \\
  Tsinghua University, Beijing, China \\
  \texttt{zhongyithu@tsinghua.edu.cn} \\
}

\begin{document}

\maketitle

\begin{abstract}
Safety post-training can improve the harmfulness and policy compliance of Large Language Models (LLMs), but it may also reduce general utility, a phenomenon often described as the \emph{alignment tax}. We study this trade-off through the lens of continual learning: sequential alignment stages expose the model to shifted data distributions and objectives, and their gradients may interfere with directions that support previously acquired general capabilities. This view does not claim that all alignment degradation has a single cause; rather, it provides a useful first-order mechanism for mitigating one important source of capability regression. We propose \textbf{O}rthogonal \textbf{G}radient \textbf{P}rojection for \textbf{S}afety \textbf{A}lignment (\textbf{OGPSA}), a lightweight update rule that estimates a low-rank reference subspace from gradients on a small set of general-capability data and removes from each safety gradient the component lying in this subspace. The resulting update is the steepest local safety-descent direction subject to first-order preservation constraints on the reference objectives. OGPSA is compatible with standard post-training pipelines and avoids large-scale replay, although it introduces periodic reference-gradient computation. Across Supervised Fine-Tuning (SFT), Direct Preference Optimization (DPO), and sequential SFT$\rightarrow$DPO settings, OGPSA improves the observed safety--utility trade-off over standard baselines. Under the sequential SFT$\rightarrow$DPO pipeline, the average performance gain increases from 33.98\% to 42.74\% on Qwen2.5-7B-Instruct and from 19.74\% to 32.98\% on Llama3.1-8B-Instruct. We have open sourced our code at \url{https://github.com/SunGL001/OGPSA}.
\end{abstract}

\begin{figure}[h]
    \centering
    \includegraphics[width=0.75\linewidth]{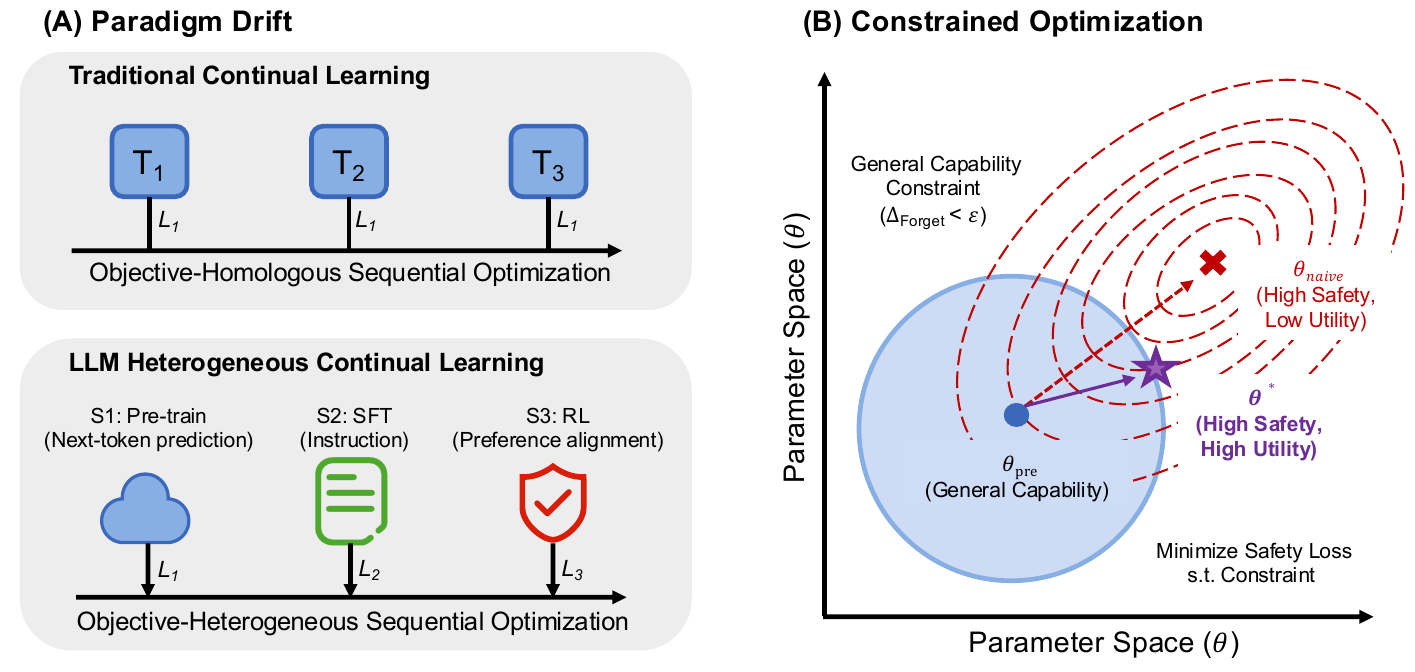}
    \vspace{-.1cm}
    \caption{Conceptual framework for reframing LLM Safety Alignment as a Constrained Continual Learning Problem.
    (A) Comparison of traditional CL and LLM Heterogeneous CL. 
    (B) Safety alignment under anti-forgetting constraints. 
    }
    \label{fig:diagram}
    \vspace{-5ex}
\end{figure}

\section{Introduction}\label{sec:intro}

Large Language Models (LLMs) have emerged as highly capable general-purpose systems~\citep{achiam2023gpt, bai2023qwen, dubey2024llama}, achieving strong performance in complex reasoning~\citep{cobbe2021training,hendrycks2measuring}, code generation~\citep{chen2021evaluating,nam2024using}, and open-ended content synthesis \citep{sudhakaran2023mariogpt,kantharaj2022opencqa,liu2025scientific}. 
However, capability alone does not imply safe or aligned behavior: without explicit alignment, LLMs may generate toxic or biased outputs, produce persuasive misinformation, or provide assistance that enables harmful actions~\citep{dongrobust,liu2023jailbreaking,wangdecodingtrust}. 
As a result, safety and reliability have become central requirements for deployment, often summarized by the desiderata of being \textbf{helpful, honest, and harmless} (HHH)~\citep{ouyang2022training}.

In practice, safety alignment is typically implemented via a dedicated \emph{post-training} pipeline~\citep{wang2024comprehensive}. 
After large-scale pre-training endows broad general capabilities, the model is further optimized to follow human intent and safety constraints using \emph{Supervised Fine-Tuning (SFT)}~\citep{bianchisafety,choi2024safety} and/or preference-based optimization such as \emph{RLHF}~\citep{ouyang2022training,daisafe} or \emph{Direct Preference Optimization (DPO)} \citep{rafailov2023direct}. 
While effective at reducing harmful behaviors, this sequential optimization frequently incurs an \textbf{alignment tax}: improving safety can lead to measurable regressions in general capabilities (e.g., truthfulness or general helpfulness, see naive tuning in Fig.~\ref{fig:main_results})~\citep{ouyang2022training,askell2021general,noukhovitch2023language}. 
One important mechanism is \emph{parameter interference across stages}: updates induced by safety objectives can overlap with directions that support pre-trained competencies, yielding capability loss even as safety improves~\citep{kirkunderstanding,lin2024mitigating}. We do not claim that this mechanism exhausts all sources of alignment tax; data curation, objective misspecification, refusal calibration, optimizer settings, and benchmark sensitivity can also contribute. Our focus is the gradient-interference component because it admits a simple, local intervention.

Recent work attempts to mitigate this trade-off by \emph{anchoring} post-training updates to the pre-trained model through two common mechanisms. 
First, \emph{rehearsal/replay} interleaves a subset of general data or auxiliary pre-training-style objectives during alignment (e.g., \textsc{PPO-ptx} in InstructGPT~\citep{ouyang2022training}), which can reduce regressions but increases compute and introduces additional scheduling and mixture hyperparameters~\citep{lin2024mitigating}. 
Second, \emph{proximity regularization} constrains the aligned policy to remain close to a reference model, most prominently via KL penalties in PPO-style RLHF and related preference-optimization objectives \citep{papineni2002bleu,yangadamerging,huang2021continual}. 
Although these techniques often improve capability retention, they can introduce additional burdens, including elevated data requirements, pipeline complexity, and sensitivity to hyperparameters such as the replay ratio or KL penalty~\citep{zhangstair,lin2024mitigating}.
More fundamentally, they act as \emph{soft constraints}: they shrink the overall update or penalize distributional deviation, but do not explicitly remove the components of the safety update that interfere with capability-preserving directions in parameter space. 
Consequently, safety gradients may still project onto subspaces that encode pre-trained competencies, leading to \emph{(catastrophic) forgetting}---a measurable drop in performance on previously acquired general skills after alignment.

\begin{wrapfigure}{r}{0.50\textwidth}
    \centering
    \vspace{-0.4cm}
    \includegraphics[width=0.99\linewidth]{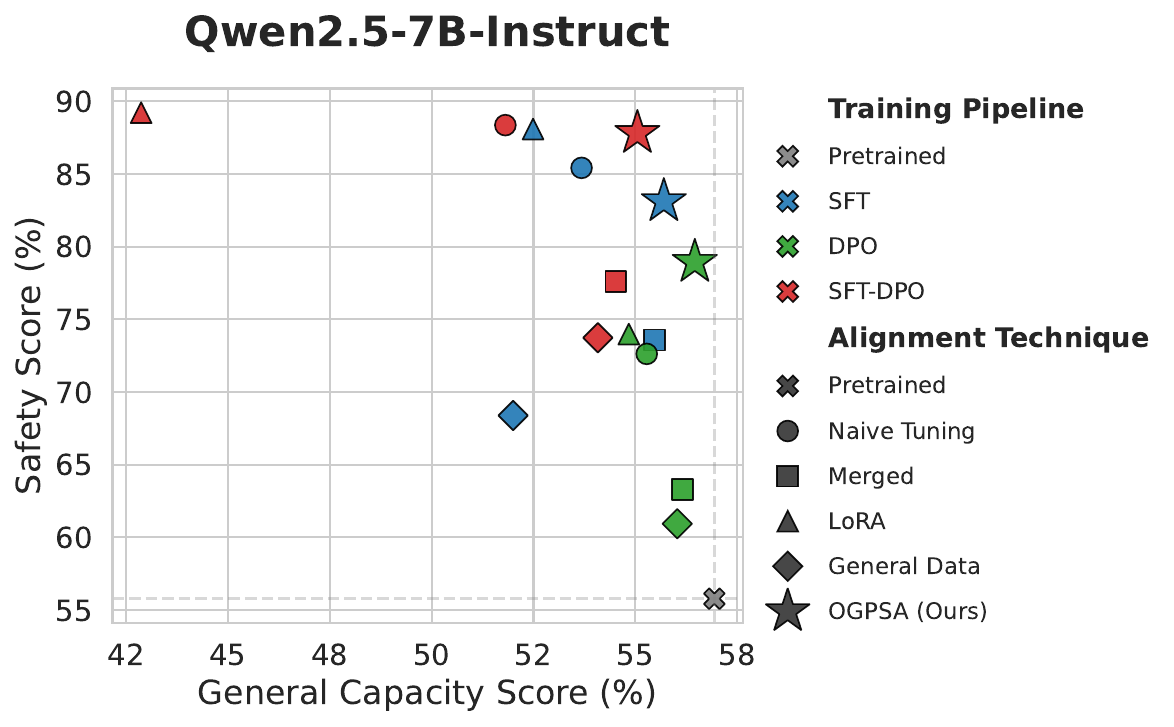}
    \vspace{-0.4cm}
    \caption{Overall performance of alignment strategies on Qwen2.5-7B-Instruct.
    We report the aggregate \textit{Safety Score} (avg. of 4 datasets) and \textit{General Capability Score} (avg. of 6 datasets); see Table~\ref{tab:main} for details and Appendix~\ref{fig:llama_results} for Llama3.1-8B results. 
    }
    \label{fig:main_results}
    \vspace{-0.3cm}
\end{wrapfigure}


To move beyond heuristic anchoring, we interpret a substantial part of the alignment tax as catastrophic-forgetting-like interference under \emph{objective-heterogeneous} sequential optimization (Fig.~\ref{fig:diagram}A).
This yields a key observation specific to modern LLM alignment: post-training is inherently a \textbf{Continual Learning (CL)} process, where the model is updated across multiple training stages (e.g., SFT followed by preference optimization) that induce heterogeneous shifts in \emph{both} data distributions and optimization objectives \citep{ouyang2022training,lin2024mitigating}.
From the perspective of CL, safety-induced gradients may overlap with parameter directions that are important for general capabilities. This fundamental conflict mirrors the classic \emph{stability--plasticity dilemma}~\citep{wang2024comprehensive_cl,zhou2024continual}: effective alignment demands the \emph{plasticity} to acquire new safety constraints without compromising the \emph{stability} of pre-trained general knowledge (Fig.~\ref{fig:diagram}B).
Accordingly, the core challenge is not merely to regularize the update magnitude, but to design updates that satisfy safety objectives \emph{while explicitly minimizing interference} with the parameter subspaces that support general capabilities.

To bridge this gap, we introduce a first-order constrained optimization view of safety post-training. We propose \textbf{O}rthogonal \textbf{G}radient \textbf{P}rojection for \textbf{S}afety \textbf{A}lignment (\textbf{OGPSA}, Fig.~\ref{fig:projection}), a lightweight geometric procedure that reduces directional interference between safety-driven updates and a reference subspace associated with general capabilities. OGPSA uses a small, representative subset of general data to estimate a low-rank gradient subspace. During alignment (e.g., via SFT or DPO), the method projects each safety gradient onto the orthogonal complement of this subspace. This operation removes the component of the safety update that would increase the selected reference losses to first order, while keeping the remaining component available for safety optimization. Empirically, OGPSA improves the observed safety--capability trade-off relative to standard baselines across multiple models, benchmarks, and alignment stages (Fig.~\ref{fig:main_results}, Table~\ref{tab:main}).

Our main contributions are summarized as follows:
\vspace{-.3cm}
\begin{itemize}
    \item We formulate safety post-training as an objective-heterogeneous continual learning problem and identify gradient interference as a concrete, testable mechanism behind part of the alignment tax.
    \item We propose OGPSA, a plug-and-play gradient projection rule that updates along the orthogonal complement of a low-rank general-capability reference subspace, with a first-order feasible-descent characterization.
    \item We evaluate OGPSA across model families and alignment strategies, showing consistent improvements in the empirical safety--utility trade-off over standard baselines while reporting the limitations of the first-order approximation.
\end{itemize}

\section{Related Work}\label{sec:related}

\paragraph{LLM Safety Alignment.}

Research on safety alignment for LLMs primarily centers on two perspectives. 
The first line of work involves \textit{test-time intervention}, which introduces external safety guards to identify unsafe responses~\citep{inan2023llama, leeharmaug,jaech2024openai,wang2024self} or actively adjusts the output distribution via model steering~\citep{kowsher2025propulsion,rebedea2025guardrails,wu2025automating}. 
However, these approaches invariably incur additional inference latency and increase system complexity. 
The second perspective focuses on \textit{post-training} the model for safety awareness. 
Nevertheless, simply training the model on safety data often leads to a degradation in general capabilities~\citep{ouyang2022training,askell2021general,noukhovitch2023language}. 
Existing methods attempt to mitigate this by introducing replay data to preserve original abilities or designing task-specific pipelines~\citep{ouyang2022training,lin2024mitigating,zhangstair}. 
Yet, the former solution significantly increases training computational costs, while the latter complicates the training pipeline and lacks universality across different training pipeline~\citep{wang2024comprehensive}. 
Moreover, both solutions are largely heuristic, lacking theoretical guarantees for the training outcomes~\citep{lin2024mitigating}. 
In contrast, our method adapts the gradient-projection principle from continual learning to objective-heterogeneous LLM safety alignment. It provides a first-order characterization of the safety-descent direction under reference-preservation constraints, rather than a global guarantee of safety or capability preservation. Its implementation is lightweight relative to large-scale replay, while still requiring periodic reference-gradient computation.

\vspace{-.3cm}

\paragraph{Continual Learning.} Continual Learning (CL) aims to enable models to learn sequential tasks without suffering from catastrophic forgetting, addressing the classic stability-plasticity dilemma~\citep{wang2024comprehensive_cl,zhou2024continual}. Traditional CL methods generally fall into three categories:
(1) \textit{Regularization-based methods} which impose penalty terms on important parameters to restrict their changes (e.g., EWC~\citep{kirkpatrick2017overcoming}, LwF~\citep{li2017learning}) ;
(2) \textit{Replay-based methods} which retain a buffer of historical data for rehearsal (e.g., GEM~\citep{lopez2017gradient}, DER~\citep{buzzega2020dark}); and
(3) \textit{Optimization-based methods} which decouple parameter updates at the gradient level to facilitate the learning of new tasks while effectively preserving pre-existing knowledge~\citep{lu2024visual,qiao2025gradient,lin2022trgp}.
More recently, advanced CL methods have shifted toward leveraging pretrained models via parameter-efficient tuning~\citep{wang2022learning,wu2025sdlora} and representation alignment~\citep{zhang2023slca,mcdonnell2024ranpac} to achieve superior rehearsal-free performance. Since safety alignment shares with CL the goal of learning new behavior without erasing useful prior behavior, it can benefit from CL concepts such as stability--plasticity trade-offs and gradient interference. 
However, while effective in standard settings, most existing CL research assumes a sequence of tasks with a \textit{homogeneous} optimization objective (e.g., a sequence of classification tasks) where only the data distribution shifts. In contrast, the LLM training lifecycle involves a \textit{multi-stage} process where both the data distribution and the optimization objective shift drastically~\citep{ouyang2022training,lin2024mitigating}.
Consequently, directly applying traditional CL methods is non-trivial: the reference behavior to preserve is broad and multi-domain, while the new objective may be likelihood-based, preference-based, or a sequence of both. OGPSA is tailored to this setting by constructing the preserved subspace from general-capability reference gradients and applying the projection inside standard SFT/DPO-style updates.

\paragraph{Positioning relative to gradient-projection CL.}
Unlike traditional projection-based CL (e.g., GEM~\citep{gem}, GPM~\citep{gpm}) that protects specific prior tasks under homogeneous losses, OGPSA is explicitly designed for objective heterogeneity. It preserves broad LLM capabilities across diverse alignment stages (SFT, DPO, SFT$\rightarrow$DPO).
Thus, our contribution lies not in the projection operator itself, but in its tailored formulation, subspace construction, and validation for safety alignment under objective heterogeneity.



\section{Preliminaries}
\label{sec:prelim}

We study \emph{sequential post-training} for safety alignment and its tendency to reduce general utility (the \emph{alignment tax}).
We first define the alignment tax at the evaluation level, then introduce a differentiable reference-loss surrogate that yields a tractable \emph{first-order} preservation constraint.


\subsection{Sequential Safety Alignment and the Alignment Tax}
\label{sec:prelim_tax}

Let $\theta_{\mathrm{pre}}$ denote the parameters of a pre-trained LLM trained on a broad next-token objective.
Safety alignment then applies one or more post-training stages (e.g., SFT, DPO~\citep{rafailov2023direct}), producing $\theta_{\mathrm{safe}}$.
While these stages can improve safety behavior, they may degrade general utility.

Let $\Phi(\theta;\mathcal{D}_{\mathrm{eval}})$ be an evaluation metric on a general evaluation suite $\mathcal{D}_{\mathrm{eval}}$.
We define the alignment tax as
\begin{equation}
\label{eq:tax_def}
\Delta_{\mathrm{tax}}
=
\Phi(\theta_{\mathrm{pre}};\mathcal{D}_{\mathrm{eval}})
-
\Phi(\theta_{\mathrm{safe}};\mathcal{D}_{\mathrm{eval}}).
\end{equation}
In practice, directly constraining $\Phi$ during training is difficult (often non-differentiable or expensive) so we introduce a differentiable \emph{capability surrogate}.

\subsection{Heterogeneous Continual Learning Perspective}
\label{sec:prelim_cl}

We model safety alignment as \emph{heterogeneous continual learning} (HCL) because the post-training pipeline is \emph{sequential} and each stage typically changes both the \emph{data distribution} and the \emph{objective} (Fig.~\ref{fig:diagram}A)~\citep{ouyang2022training,lin2024mitigating}.
Starting from a pre-trained model $\theta_{\mathrm{pre}}$ learned on a broad pre-training distribution, alignment proceeds
through stages such as instruction tuning and preference optimization, 
e.g., SFT and DPO on safety dataset $\mathcal{D}_{\mathrm{safe}}$.
Importantly, these stages do not merely introduce new samples; they can also alter the risk functional---for example, from likelihood-based supervision to preference/ranking-based optimization---which can substantially reshape gradient geometry.

Consider a generic alignment stage that optimizes a safety-related objective $\mathcal{L}_{\mathrm{safe}}(\theta)$
(e.g., SFT or the DPO~\citep{rafailov2023direct} loss). A standard gradient update takes the form
\begin{equation}
\label{eq:naive_update}
\theta \leftarrow \theta - \eta\, g_{\mathrm{safe}},
\qquad
g_{\mathrm{safe}} := \nabla_\theta \mathcal{L}_{\mathrm{safe}}(\theta),
\end{equation}
where $\eta$ is learning rate. Under HCL, one source of alignment tax can be interpreted as continual-learning-style interference: due to distribution and objective shifts across stages, $g_{\mathrm{safe}}$ may contain components along parameter directions that are also important for general capabilities acquired during pre-training. Consequently, the naive update in Eq.~\eqref{eq:naive_update} can improve safety behavior while perturbing capability-supporting directions, yielding degradation in general utility (Fig.~\ref{fig:diagram}B).

\subsection{First-Order Capability Preservation via Gradient Orthogonality}
\label{sec:prelim_first_order}

Motivated by evidence that fine-tuning often operates in low-dimensional effective subspaces~\citep{aghajanyan2021intrinsic,zhou2023lima,ying2026truthfulness}, we approximate capability preservation by estimating a low-rank gradient subspace from a small \emph{reference} collection of general-purpose data. Let $\{\mathcal{D}^{(i)}_{\mathrm{ref}}\}_{i=1}^{M}$ be $M$ small
datasets, each targeting a facet of general ability (e.g., reasoning, coding, truthfulness). Let
$\mathcal{L}^{(i)}_{\mathrm{ref}}(\theta)$ denote a differentiable loss on $\mathcal{D}^{(i)}_{\mathrm{ref}}$ (e.g.,
cross-entropy), and define the corresponding reference gradients
\begin{equation}
\label{eq:ref_grad}
g^{(i)}(\theta) := \nabla_\theta \mathcal{L}^{(i)}_{\mathrm{ref}}(\theta), \qquad i=1,\dots,M.
\end{equation}

Consider a small parameter update $\Delta\theta$. A first-order Taylor expansion gives
\begin{equation}
\label{eq:taylor}
\mathcal{L}^{(i)}_{\mathrm{ref}}(\theta+\Delta\theta)
\approx
\mathcal{L}^{(i)}_{\mathrm{ref}}(\theta)
+
\langle g^{(i)}(\theta), \Delta\theta\rangle.
\end{equation}
Thus, a sufficient condition to preserve reference capability $i$ \emph{to first order} is
$\langle g^{(i)}(\theta), \Delta\theta\rangle = 0$. Enforcing this for all $i$ yields the linear constraints
\begin{equation}
\label{eq:linear_constraints}
\langle g^{(i)}(\theta), \Delta\theta\rangle = 0,\quad i=1,\dots,M.
\end{equation}

We summarize these directions via the \emph{general-capability subspace}
\begin{equation}
\label{eq:subspace_def}
\mathcal{S}_{\mathrm{gen}}(\theta) := \mathrm{span}\{g^{(1)}(\theta),\dots,g^{(M)}(\theta)\}.
\end{equation}
Equation~\eqref{eq:linear_constraints} is equivalent to requiring $\Delta\theta \in \mathcal{S}_{\mathrm{gen}}(\theta)^\perp$.
This yields the first-order update rule behind our method: \emph{remove from the safety update the component that lies in the local general-capability reference subspace.} The next section operationalizes this principle by maintaining a low-rank basis
for $\mathcal{S}_{\mathrm{gen}}(\theta)$ and projecting each safety gradient accordingly, resulting in an efficient
plug-and-play update rule.

\vspace{-.3cm}
\section{Methodology}
\label{sec:method}
\vspace{-.3cm}
In this section, we present \textbf{O}rthogonal \textbf{G}radient \textbf{P}rojection for \textbf{S}afety \textbf{A}lignment (\textbf{OGPSA}, Fig.~\ref{fig:projection}). OGPSA is a plug-and-play update rule that reduces first-order gradient interference between safety optimization and selected general-capability reference objectives. It estimates a low-rank reference subspace from general-capability gradients and projects each safety gradient onto the orthogonal complement of this subspace before updating parameters.

\vspace{-.2cm}
\subsection{Overview}
\label{sec:method_overview}

\begin{wrapfigure}{r}{0.5\textwidth}
    \centering
    \vspace{-1.2cm}
    \includegraphics[width=0.85\linewidth]{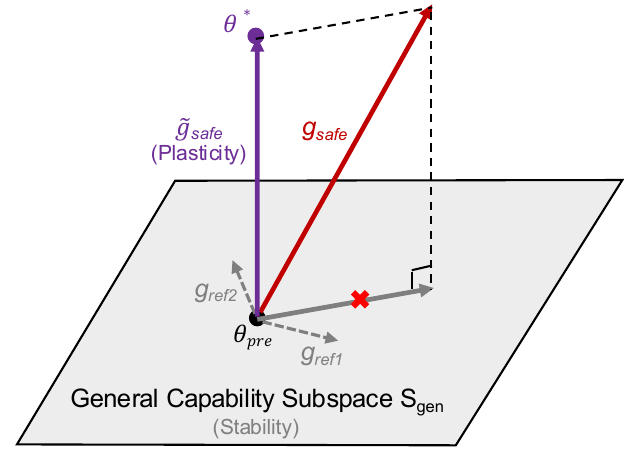}
    \vspace{-0.1cm}
    \caption{Schematic illustration of the proposed {O}rthogonal \textbf{G}radient \textbf{P}rojection for \textbf{S}afety \textbf{A}lignment (OGPSA) framework.
    $g_{\text{ref1}}, g_{\text{ref2}}$: Reference gradients computed from representative general capability datasets (e.g., helpfulness, truthfulness). 
    $g_{\text{safe}}$: The standard gradient derived from the safety alignment objective. 
    $\tilde{g}_{\text{safe}}$: The projected safety gradient obtained by projecting $g_{\text{safe}}$ onto the orthogonal space of the general capability subspace.
    }
    \label{fig:projection}
    \vspace{-0.8cm}
\end{wrapfigure}

Modern alignment is typically performed \emph{sequentially} after pre-training, and often across multiple stages with
shifting objectives and data distributions (e.g., likelihood-based SFT on $\mathcal{D}_{\mathrm{sft}}$ followed by
preference optimization on $\mathcal{D}_{\mathrm{safe}}$). This setting is naturally viewed as
\emph{heterogeneous continual learning}, where both the task objective and the training distribution change over time.
Consequently, a naive safety update through Eq.~\ref{eq:naive_update} can interfere with parameter directions that are important for broad utility, inducing continual-learning-style capability regression.

OGPSA constrains each safety step to avoid directions that locally encode general capability. Concretely, we maintain a
low-rank \emph{general-capability subspace} $\mathcal{S}_{\mathrm{gen}}(\theta)$ estimated from reference gradients
$\{g^{(i)}(\theta)\}_{i=1}^{M}$ computed on small, diverse general-capability datasets. We then update parameters using
only the component of the safety gradient orthogonal to this subspace:
\begin{equation}
\label{eq:proj_rule_method_overview}
\Delta\theta = -\eta\, P_{\mathcal{S}_{\mathrm{gen}}(\theta)^\perp}\!\big(g_{\mathrm{safe}}(\theta)\big),
\qquad
\theta \leftarrow \theta + \Delta\theta.
\end{equation}
Equivalently, letting $U$ denote an orthonormal basis of $\mathcal{S}_{\mathrm{gen}}(\theta)$ (rank $M'$), the projected
direction is $\tilde g_{\mathrm{safe}} = g_{\mathrm{safe}} - U(U^\top g_{\mathrm{safe}})$, and we take
$\theta \leftarrow \theta - \eta\, \tilde g_{\mathrm{safe}}$.

The subspace is refreshed periodically (every $K$ steps) using inexpensive reference mini-batches, and the projection
requires only a small number of inner products for low rank $M'$. As a result, OGPSA can be applied across alignment stages (e.g., SFT/DPO/RLHF-style updates) without modifying the underlying objective, while adding only periodic reference-gradient computation and low-rank projection operations.

We next describe dynamic subspace construction, the projected update rule with its first-order justification, and the
resulting algorithm and computational overhead (see Fig.~\ref{fig:projection} and Algorithm~\ref{alg:ogpsa}).

\subsection{General-Capability Subspace Estimation}
\label{sec:method_subspace_est}

Directly constraining general-utility metrics during training is typically infeasible because such metrics are often non-differentiable, benchmark-specific, or too expensive to evaluate at every step. Instead, we approximate capability
preservation using a small set of differentiable \emph{reference} objectives~\citep{aghajanyan2021intrinsic,zhou2023lima}. Let
$\{\mathcal{D}^{(i)}_{\mathrm{ref}}\}_{i=1}^{M}$ be $M$ small datasets, each targeting one facet of general capability
(e.g., reasoning, coding, truthfulness). For each dataset, we define a differentiable loss
$\mathcal{L}^{(i)}_{\mathrm{ref}}(\theta)$ (e.g., cross-entropy) and its gradient
\begin{equation}
\label{eq:ref_grad_method}
g^{(i)}(\theta) := \nabla_\theta \mathcal{L}^{(i)}_{\mathrm{ref}}(\theta), \qquad i=1,\dots,M.
\end{equation}
We define the \emph{general-capability subspace} as the span of these gradients:
\begin{equation}
\label{eq:subspace_def_method}
\mathcal{S}_{\mathrm{gen}}(\theta) := \mathrm{span}\{g^{(1)}(\theta),\dots,g^{(M)}(\theta)\}.
\end{equation}

\paragraph{Dynamic, low-rank basis.}

Since the local geometry can shift as training progresses, we update the subspace periodically. 
Every $K$ steps (i.e., at step $\tau$), we compute $M$ reference gradients on mini-batches $B^{(i)}\sim \mathcal{D}^{(i)}_{\mathrm{ref}}$ and construct an orthonormal basis $U_\tau=[u_{1},\dots,u_{M'}]\in\mathbb{R}^{d\times M'}$ for $\mathcal{S}_{\mathrm{gen}}(\theta_\tau)$. $M'$ denotes the rank of the estimated subspace, where $M' \le M$ accounts for the potential removal of linearly dependent directions. 
We employ the Gram--Schmidt process~\citep{bjorck1994numerics,leon2013gram} with a threshold $\delta$ to filter out redundancy:


\vspace{-.5cm}

\begin{align}
\label{eq:gs1}
u_{\tau,1} &= \frac{g_{\tau}^{(1)}}{\|g_{\tau}^{(1)}\|+\epsilon},\\
\label{eq:gsk}
    v_{\tau,k} &= g_{\tau}^{(k)} - \sum_{j=1}^{k-1}\langle g_{\tau}^{(k)}, u_{\tau,j}\rangle u_{\tau,j}, \text{ and }
    u_{\tau,k} = \frac{v_{\tau,k}}{\|v_{\tau,k}\|+\epsilon}\quad \text{if }\|v_{\tau,k}\|\ge \delta,
\end{align}
discarding nearly collinear directions when $\|v_{\tau,k}\|<\delta$.


\subsection{Projected Safety Optimization}
\label{sec:method_proj_update}

At training iteration $t$, let $g_{\mathrm{safe}} := \nabla_\theta \mathcal{L}_{\mathrm{safe}}(\theta_t)$ denote the safety
gradient. OGPSA maintains a (lagged) orthonormal basis $U_\tau=[u_{\tau,1},\dots,u_{\tau,M'}]\in\mathbb{R}^{d\times M'}$
for the current general-capability subspace $\mathcal{S}_{\mathrm{gen}}(\theta)\approx \mathrm{span}\{g^{(i)}(\theta)\}_{i=1}^M$,
refreshed every $K$ steps (so $\tau=\lfloor t/K\rfloor$). We remove the components of $g_{\mathrm{safe}}$ that lie in
$\mathcal{S}_{\mathrm{gen}}(\theta)$ by projecting onto its orthogonal complement:
\begin{equation}
\label{eq:proj_grad}
\tilde{g}_{\mathrm{safe}}
=
g_{\mathrm{safe}} - U_\tau(U_\tau^\top g_{\mathrm{safe}})
=
g_{\mathrm{safe}} - \sum_{j=1}^{M'} \langle g_{\mathrm{safe}}, u_{\tau,j}\rangle u_{\tau,j}.
\end{equation}
We then perform the projected update
\begin{equation}
\label{eq:update_final}
\theta_{t+1} \leftarrow \theta_t - \eta\, \tilde{g}_{\mathrm{safe}}.
\end{equation}
Intuitively, \eqref{eq:proj_grad}--\eqref{eq:update_final} make each projected safety step lie in $\mathcal{S}_{\mathrm{gen}}(\theta)^\perp$ up to refresh lag and stochastic gradient noise, thereby reducing first-order interference with the selected reference directions.

\begin{table*}[t]
\centering
\caption{Comparative evaluation of safety alignment and general capability retention. We compare our method (OGPSA) against standard baselines (SFT, DPO~\citep{rafailov2023direct}, SFT-DPO) and mitigation strategies (Merge, LoRA, Data Mixing) across Llama3.1-8B-Instruct and Qwen2.5-7B-Instruct Model. 
}
\resizebox{1.00\textwidth}{!}{
    \begin{tabular}{lccccccccccc}
        \toprule
\multirow{2}{*}{\textbf{Model}}
         & \multicolumn{4}{c}{\textbf{Safety ($\uparrow$)}} & \multicolumn{3}{c}{\textbf{Truthful ($\uparrow$)}} & \multicolumn{2}{c}{\textbf{Helpful ($\uparrow$)}} & \textbf{Robustness ($\uparrow$)} &\multirow{2}{*}{\textbf{Avg. Gain ($\uparrow$)}} \\
        \cmidrule(lr){2-5} \cmidrule(lr){6-8} \cmidrule(lr){9-10} \cmidrule(lr){11-11}
        & XSTest & WildChat & Stereotype & StrongReject & SimpleQA	& GPQA &MMLU & IFEval & HHH & AdvGLUE \\    
        \midrule
        \multicolumn{11}{c}{\textbf{Llama3.1-8B-Instruct Model}} \\
        \midrule
        Instruct Baseline  &88.50	&15.80	&89.85	&56.59	&1.90	&22.90	&71.07	&67.28	&86.61	&46.30 &- \\
        \midrule
        \textbf{SFT}     &83.00	&\textbf{50.00}	&89.08	&\textbf{97.27}	&0.72	&24.58	&70.00	&57.12	&87.89	&41.92 &\underline{20.19} \\
        + Merge &79.00	&31.20	&\underline{92.02}	&77.34	&1.43	&23.91	&\underline{70.64}	&\textbf{63.96}	&\underline{87.48}	&\textbf{44.92} &9.80 \\
        + LoRA   &\underline{85.99}	&42.60	&85.63	&\textbf{96.62}	&0.42	&\underline{24.92}	&69.07	&51.20	&87.48	&40.73 &12.60\\
        + General Data    &85.50	&14.40	&89.46	&71.38	&\textbf{3.54}	&23.40	&67.43	&58.60	&83.74	&42.52 &7.24 \\
        + OGPSA (Ours)  &\textbf{86.00}	&\underline{47.00}	&\textbf{92.53}	&94.81	&\underline{2.88}	&\textbf{25.59}	&\textbf{70.93}	&\underline{60.26}	&\textbf{89.18}	&\underline{43.59} & \textbf{31.50} \\
        \midrule
        \textbf{DPO}  &81.00	&\underline{42.80}	&70.31	&\textbf{97.66}	&0.25	&23.57	&66.21	&41.22	&77.92	&33.17 &4.54 \\
        + Merged  &\underline{88.00}	&30.80	&\underline{88.89}	&73.22	&1.27	&\textbf{25.42}	&70.36	&63.36	&84.97	&45.40  &8.99  \\
        + LoRA   &82.50	&\textbf{43.80}	&72.41	&\underline{87.35}	&0.62	&21.72	&69.64	&58.99	&\underline{85.36}	&43.05 &\underline{11.01} \\
        + General Data     &63.00	&15.40	&\textbf{90.42}	&57.28	&\underline{2.59}	&\underline{24.24}	&\underline{71.36}	&\textbf{64.51}	&\textbf{87.09}	&\textbf{50.30} &1.82\\
        + OGPSA (Ours) &\textbf{95.00}	&38.40	&75.86	&85.89	&\textbf{3.05}	&23.57	&\textbf{71.57}	&\underline{63.77}	&83.71	&\underline{49.59}  &\textbf{24.93} \\
        \midrule
        \textbf{SFT-DPO}    &87.00	&\underline{59.40}	&72.61	&\underline{99.87}	&0.28	&26.94	&69.00	&43.25	&84.53	&34.51 &\underline{19.74} \\
        + Merged   &82.00	&36.00	&\underline{89.46}	&83.29	&1.04	&24.92	&\textbf{70.21}	&\textbf{62.11}	&\textbf{87.44}	&43.67  &11.72\\
        + LoRA   &81.50	&\textbf{80.80}	&58.05	&\textbf{99.93}	&0.07	&10.27	&61.21	&16.82	&82.04	&26.29 &15.58 \\
        + General Data   &\underline{89.00}	&17.60	&\textbf{94.44}	&75.73	&\textbf{4.05}	&\textbf{29.46}	&66.71	&\underline{59.15}	&84.15	&\textbf{44.71} &16.82 \\
        + OGPSA (Ours)  &\textbf{91.50}	&45.80	&88.12	&94.48	&\underline{3.28}	&\underline{27.95}	&\underline{70.00}	&57.67	&\underline{84.98}	&\underline{43.73} &\textbf{32.98}\\
        \midrule
        \multicolumn{11}{c}{\textbf{Qwen2.5-7B-Instruct Model}} \\
        \midrule
        Instruct Baseline &65.50	&16.00	&96.74	&44.83	&3.33	&34.18	&73.50	&64.33	&88.77	&77.55 &-\\
        \midrule
        \textbf{SFT}     &87.00	&\underline{64.20}	&\textbf{100.00}	&\underline{90.48}	&0.79	&\underline{34.18}	&72.00	&57.30	&88.34	&69.48 &33.91 \\
        + Merged   &83.00	&45.60	&99.04	&66.65	&2.24	&34.04	&\underline{72.64}	&\underline{61.74}	&88.34	&\underline{73.82} &21.91 \\
        + LoRA   &\textbf{88.00}	&\textbf{68.80}	&\textbf{100.00}	&\textbf{95.54}	&0.67	&32.83	&71.50	&51.39	&\underline{88.77}	&69.78  &\underline{36.42}\\
        + General Data       &77.50	&26.60	&99.04	&70.38	&\underline{3.47}	&23.57	&69.93	&53.42	&86.28	&\textbf{75.35} &8.96 \\
        + OGPSA (Ours)    &\underline{87.50}	&57.60	&\textbf{100.00}	&87.37	&\textbf{3.61}	&\textbf{34.68}	&\textbf{73.21}	&\textbf{63.03}	&\textbf{89.97}	&69.72 &\textbf{39.06}\\
        \midrule
        \textbf{DPO}       &\underline{87.00}	&41.40	&98.47	&63.60	&1.57	&31.31	&\underline{73.43}	&\underline{63.22}	&87.52	&74.60 &\underline{16.69}\\
        + Merged  &79.50	&26.20	&97.32	&50.12	&2.59	&31.99	&\textbf{74.00}	&62.48	&\underline{89.62}	&\underline{76.25} &6.60  \\
        + LoRA   &85.50	&\underline{42.80}	&\underline{98.66}	&\underline{68.96}	&0.97	&31.14	&73.36	&61.55	&89.21	&72.80 &16.39 \\
        + General Data      &78.50	&21.00	&97.32	&46.89	&\underline{3.29}	&\underline{32.15}	&73.21	&62.85	&87.92	&\textbf{76.73} & 4.44 \\
        + OGPSA (Ours)   &\textbf{94.50}	&\textbf{49.40}	&\textbf{99.23}	&\textbf{72.64}	&\textbf{3.35}	&\textbf{34.51}	&72.79	&\textbf{63.40} &\textbf{90.68}	&74.05 &\textbf{31.44} \\
        \midrule
        \textbf{SFT-DPO}   &\textbf{92.50}	&\underline{65.80}	&\underline{99.62}	&95.53	&0.53	&30.98	&\underline{72.14}	&51.94	&88.36	&66.92 &\underline{33.98} \\
        + Merged   &87.50	&50.00	&99.43	&73.44	&1.62	&\underline{31.08}	&\textbf{72.93}	&\underline{60.44}	&88.75	&72.28  &23.86\\
        + LoRA   &\underline{90.50}	&\textbf{69.85}	&96.55	&\textbf{100.00}	&0.00	&21.04	&64.79	&29.39	&78.81	&63.18 &26.32 \\
        + General Data     &87.50	&35.20	&99.04	&73.16	&\textbf{3.49}	&28.45	&70.86	&56.93	&\textbf{89.83}	&\textbf{74.92} &18.99 \\
        + OGPSA (Ours)  &\underline{90.50}	&63.60	&\textbf{100.00}	&\underline{97.14}	&\underline{3.03}	&\textbf{31.31}	&70.64	&\textbf{63.96}	&\underline{88.98}	&\underline{72.35} &\textbf{42.74}\\
        \bottomrule
    \end{tabular}
}
\vspace{-3ex}
\label{tab:main}
\end{table*}

\paragraph{First-order preservation and feasible descent.}
\label{sec:method_theory}
We justify the projection rule via a first-order preservation argument. Consider a local parameter perturbation
$\Delta\theta$. For each reference objective, a first-order expansion yields
\begin{equation}
\label{eq:taylor_method}
\mathcal{L}^{(i)}_{\mathrm{ref}}(\theta+\Delta\theta)
\approx
\mathcal{L}^{(i)}_{\mathrm{ref}}(\theta)
+
\langle g^{(i)}(\theta), \Delta\theta\rangle.
\end{equation}
Thus, a sufficient condition to preserve reference capability $i$ to first order is
$\langle g^{(i)}(\theta), \Delta\theta\rangle = 0$. Enforcing this for all $i$ yields the linear constraints
\begin{equation}
\label{eq:linear_constraints2}
\langle g^{(i)}(\theta), \Delta\theta\rangle = 0,\quad i=1,\dots,M,
\end{equation}
equivalently $\Delta\theta \in \mathcal{S}_{\mathrm{gen}}(\theta)^\perp$. Within this local linearized constraint set, the projected gradient is the steepest instantaneous safety-descent direction. This is a local first-order statement about the chosen reference losses; it should not be read as a global guarantee that all downstream capability metrics will be preserved.

\vspace{+.5cm}

\begin{proposition}[Steepest Feasible Descent]
\label{prop:steepest}
Let $f(\theta)=\mathcal{L}_{\mathrm{safe}}(\theta)$ with gradient $g=\nabla f(\theta)$, and let
$\mathcal{S}_{\mathrm{gen}}(\theta)=\mathrm{span}\{g^{(i)}(\theta)\}_{i=1}^M$.
Among all unit vectors $v$ satisfying $\langle g^{(i)}(\theta), v\rangle=0$ for all $i$,
the maximally descending direction is
\begin{equation}
v^\star
=
-\frac{P_{\mathcal{S}_{\mathrm{gen}}(\theta)^\perp}(g)}{\|P_{\mathcal{S}_{\mathrm{gen}}(\theta)^\perp}(g)\|}.
\end{equation}
\end{proposition}

\begin{proof}
Any feasible $v$ lies in $\mathcal{S}_{\mathrm{gen}}(\theta)^\perp$.
Decompose $g=g_\parallel+g_\perp$ with $g_\parallel\in\mathcal{S}_{\mathrm{gen}}(\theta)$ and
$g_\perp\in\mathcal{S}_{\mathrm{gen}}(\theta)^\perp$.
Then $\langle g,v\rangle=\langle g_\perp,v\rangle$.
By Cauchy--Schwarz~\citep{bjorck1994numerics,leon2013gram}, the minimum of $\langle g_\perp,v\rangle$ over $\|v\|=1$ is $-\|g_\perp\|$,
achieved when $v=-g_\perp/\|g_\perp\|$.
\end{proof}

\vspace{-.5cm}
\paragraph{Algorithm and complexity.}
\label{sec:method_algo}
Algorithm~\ref{alg:ogpsa} summarizes OGPSA. The overhead consists of: (i) an additional $M$ reference-gradient computations
every $K$ steps (i.e., $M$ extra backward passes on small reference mini-batches per refresh), and (ii) a projection of
$g_{\mathrm{safe}}$ onto a rank-$M'$ subspace, which requires computing $U_\tau^\top g_{\mathrm{safe}}$ and forming
$U_\tau(U_\tau^\top g_{\mathrm{safe}})$ (equivalently $M'$ inner products plus a linear combination). For small $M'$ and moderate $K$, this overhead is substantially smaller than large-scale replay in our experiments, although it is not zero and should be reported together with wall-clock time and token counts.

\vspace{-.3 cm}

\section{Experiments}\label{sec:experiments}

\vspace{-.3 cm}

In this section, we describe the experimental setup, then present the results with an in-depth analysis.


\vspace{-.3 cm}
\subsection{Experiments Setup}
\vspace{-.3 cm}

We evaluate our framework on LLaMA3.1-8B-Instruct~\citep{dubey2024llama} and Qwen2.5-7B-Instruct~\citep{Yang2024Qwen25TR} across three standard safety alignment paradigms: SFT, DPO~\citep{rafailov2023direct}, and sequential SFT-DPO. The safety alignment utilizes a 10k-sample dataset derived from PKU-SafeRLHF~\citep{ji2024pku}. Apart from the naive tuning method, we compare OGPSA against three established mitigation strategies: (1) \textit{+Merged} (weight interpolation), (2) \textit{+LoRA} (low-rank adaptation), and (3) \textit{+General Data}, a classic replay baseline mixing 10k UltraFeedback~\citep{cui2024ultrafeedback} samples.
Further details regarding the experimental setup, including model architectures, datasets, baseline methods, evaluation metrics, and training protocols, are provided in Appendix Sec.~\ref{sec:appendix_exp}. We report individual benchmark scores as the primary evidence. The aggregate \textit{Avg. Gain} is used only as a compact summary of the safety--utility trade-off relative to the instruct baseline; because it combines heterogeneous benchmarks, conclusions should be interpreted together with the per-benchmark results rather than from the aggregate alone.

\vspace{-.3 cm}
\subsection{Overall Performance}
\vspace{-.25 cm}

\paragraph{Compare to Standard Baseline.}
As shown in Table~\ref{tab:main} and Fig.~\ref{fig:main_results}, standard safety alignment methods (SFT, DPO, and SFT-DPO) improve several safety metrics but often reduce parts of general utility, producing a visible alignment tax.
Existing mitigation strategies reduce some regressions but do not uniformly dominate the trade-off. Mixing general data (\textit{+ General Data}) or applying parameter averaging (\textit{+ Merged}) can partially recover general capabilities, but may dilute the safety signal in some settings. Parameter-efficient updating (\textit{+ LoRA}) can maintain strong safety scores, but it may still underperform on several general-capability metrics (Fig.~\ref{fig:main_results}, Table~\ref{tab:main}). 

OGPSA improves the observed trade-off by preserving more of the selected general-capability metrics while maintaining competitive safety performance. To summarize this mixed benchmark behavior, we report average performance gain (\textit{Avg. Gain}) against the instruct baseline, while retaining all individual metrics in Table~\ref{tab:main}. OGPSA obtains the highest \textit{Avg. Gain} across both models and all three alignment stages in this evaluation. In the sequential SFT-DPO pipeline, it increases \textit{Avg. Gain} on Qwen2.5-7B-Instruct from 33.98\% to 42.74\% and on Llama3.1-8B-Instruct from 19.74\% to 32.98\%.

\vspace{-.3 cm}
\paragraph{Compare to Advanced Baseline.}
We further compare against two stronger baselines in the Qwen2.5-7B SFT setting: (1) GPM~\citep{gpm}, a representative continual learning approach using gradient projection memory, and (2) STAIR~\citep{zhangstair}, a recent safety-alignment framework requiring 20K mixed data samples. This comparison is not intended to exhaust all CL or alignment baselines, but it tests whether OGPSA remains competitive against a direct projection-based CL method and a stronger safety-alignment method. As shown in Table~\ref{tab:sota_baseline}, OGPSA obtains the highest aggregate gain and improves several general-capability metrics while keeping safety scores competitive.

\begin{table*}[h]
\centering
\vspace{-0.4cm}
\caption{Comparison with advanced baselines, a continual learning baseline GPM~\citep{gpm} and a SOTA safety alignment baseline STAIR~\citep{zhangstair}, using SFT on Qwen2.5-7B-Instruct Model. 
}
\label{tab:sota_baseline}
\resizebox{1.00\textwidth}{!}{
    \begin{tabular}{lccccccccccc}
        \toprule
\multirow{2}{*}{\textbf{Model}}
         & \multicolumn{4}{c}{\textbf{Safety ($\uparrow$)}} & \multicolumn{3}{c}{\textbf{Truthful ($\uparrow$)}} & \multicolumn{2}{c}{\textbf{Helpful ($\uparrow$)}} & \textbf{Robustness ($\uparrow$)} &\multirow{2}{*}{\textbf{Avg. Gain ($\uparrow$)}} \\
        \cmidrule(lr){2-5} \cmidrule(lr){6-8} \cmidrule(lr){9-10} \cmidrule(lr){11-11}
        & XSTest & WildChat & Stereotype & StrongReject & SimpleQA	& GPQA &MMLU & IFEval & HHH & AdvGLUE \\    
        \midrule
        Qwen2.5-7B-Instruct &65.50	&16.00	&96.74	&44.83	&3.33	&34.18	&73.50	&64.33	&88.77	&77.55 &- \\
        \midrule
        \textbf{SFT}     &\underline{87.00}	&\textbf{64.20}	&\textbf{100.00}	&\textbf{90.48}	&0.79	&34.18	&72.00	&57.30	&88.34	&\underline{69.48} & \underline{33.91} \\
        +GPM~\citep{gpm}     &86.00 &\underline{61.60} &\textbf{100.00} &86.66 &1.06 &33.17 &\underline{73.00} &\underline{61.92} &\underline{89.01} &68.81 &32.64\\
        +STAIR~\citep{zhangstair}     &85.50 &29.40 &99.81 &78.29 &\textbf{4.55} &\underline{34.34} &70.07 &26.99 &83.78 &44.97 & 11.89 \\
        + OGPSA (Ours)    &\textbf{87.50}	&57.60	&\textbf{100.00}	&\underline{87.37}	&\underline{3.61}	&\textbf{34.68}	&\textbf{73.21}	&\textbf{63.03}	&\textbf{89.97}	&\textbf{69.72} & \textbf{39.06} \\
        \bottomrule
    \end{tabular}
}
\vspace{-0.4cm}
\end{table*}

\paragraph{Resistance to Optimization-Based Jailbreaks.}

\begin{wraptable}{r}{6.2 cm}
\vspace{-.65cm}
\caption{Resistance to Optimization-Based Jailbreaks (I-GCG~\citep{gcg}). 
}
\label{tab:gcg_jailbreak}
\centering
\begin{small}
\resizebox{0.95\linewidth}{!}{
\begin{tabular}{lcccc}
\toprule
\textbf{Model} & \textbf{ASR} $\downarrow$ & \textbf{Best Loss} $\uparrow$ & \textbf{$\Delta$ Loss} $\downarrow$ & \textbf{Steps to Best Loss} $\uparrow$ \\
\midrule
Qwen2.5-7B-Instruct & 82\% & 0.06 & 1.88 & 258.86 \\
\midrule
SFT                & 32\% & 0.16 & 1.77 & 264.12 \\
+ OGPSA (Ours)     & \textbf{26\%} & \textbf{0.17} & \textbf{1.33} & \textbf{269.16} \\
\midrule
DPO                & 54\% & 0.12 & 3.95 & 255.98 \\
+ OGPSA (Ours)     & \textbf{24\%} & \textbf{0.19} & \textbf{3.11} & \textbf{269.44} \\
\bottomrule
\end{tabular}
}
\end{small}
\end{wraptable}

As an initial stress test of whether the projected update creates an obvious optimization-based vulnerability, we evaluate I-GCG~\citep{gcg} on Qwen2.5-7B-Instruct. As shown in Table~\ref{tab:gcg_jailbreak}, OGPSA reduces Attack Success Rate (ASR) relative to the corresponding SFT and DPO baselines and keeps the optimization difficulty comparable or higher under this attack. This result suggests that OGPSA does not introduce an immediate vulnerability to this specific optimization-based jailbreak, although broader adaptive and multi-turn attack evaluations remain necessary.

In summary, across the evaluated models and training stages, OGPSA consistently improves the empirical safety--utility frontier relative to the standard baselines considered here. The strongest evidence is the combination of the aggregate frontier in Fig.~\ref{fig:main_results} and the per-benchmark values in Table~\ref{tab:main}; we avoid claiming that the method is Pareto-optimal outside these evaluated settings.


\vspace{-.3 cm}
\subsection{Ablations}
\vspace{-.3 cm}
We conduct ablation studies on the Qwen-2.5-7B model~\citep{Yang2024Qwen25TR} to understand the critical components of OGPSA.

\paragraph{Impact of Subspace Dimensionality and Diversity.}

\begin{wraptable}{r}{7.0 cm}
\vspace{-1.3cm}
\caption{Effect of general capability subspace composition on alignment outcomes using DPO~\citep{rafailov2023direct}.
}
\resizebox{0.5\textwidth}{!}{
    \begin{tabular}{lcccccc}
        \toprule
\multirow{2}{*}{\textbf{Model}}
         & \multicolumn{2}{c}{\textbf{Safety ($\uparrow$)}} & \multicolumn{2}{c}{\textbf{Truthful ($\uparrow$)}} & \multicolumn{2}{c}{\textbf{Helpful ($\uparrow$)}} \\
        \cmidrule(lr){2-3} \cmidrule(lr){4-5} \cmidrule(lr){6-7}
        & Stereotype  & StrongReject & SimpleQA &MMLU & IFEval & HHH \\   
        \midrule
        Qwen2.5-7B-Instruct Model &96.74		&44.83	&3.33	&73.50	&64.33	&88.77	\\ 
        \midrule
        w/o ours   &98.47	&63.60	&1.57	&73.43	&63.22	&87.52	 \\
        + w/ 1 dim (Helpful)   &98.66	&72.75	&1.94	&73.21	&62.48	&88.74 \\
        + w/ 1 dim (Truthful)    &\textbf{99.62}	&70.74	&3.17	&\textbf{74.36}	&61.74	&87.54 \\
        + w/ 1 dim (Mixed) &99.23	&\textbf{72.87}	&3.28	&73.07	&61.00	&89.21 \\
        + w/ 2 dims (Helpful$+$Truthful)   &99.23	&72.64	&\textbf{3.35}	&72.79	&\textbf{63.40}	&\textbf{90.68} \\
        \bottomrule
    \end{tabular}
}
\vspace{-0.65cm}
\label{tab:base_direction}
\end{wraptable}

Following previous studies~\citep{aghajanyan2021intrinsic,zhou2023lima,ying2026truthfulness}, we investigate subspace construction using single versus diverse domains. As shown in Table~\ref{tab:base_direction}, one-dimensional projections provide targeted protection: anchoring solely on \textit{UltraFeedback} (Helpful) boosts \textit{HHH} (88.74\%) but provides limited protection for truthfulness (\textit{SimpleQA} 1.94\%), whereas \textit{HaluEval} (Truthful) restores \textit{SimpleQA} (3.17\%) but leads to lower \textit{IFEval} and \textit{HHH}. This suggests that a single reference direction is too narrow for broad capability retention. Averaging mixed datasets into one gradient (``1 dim Mixed'') is also less effective on instruction following (\textit{IFEval} 61.00\%). In contrast, spanning a rank-2 subspace with separate Helpful and Truthful directions gives the best overall balance in this ablation (e.g., \textit{SimpleQA} 3.35\%, \textit{IFEval} 63.40\%, \textit{HHH} 90.68\%). Consistent SFT results (Appendix Tab.~\ref{tab:base_direction_sft}) and math-domain experiments (Appendix Tables~\ref{tab:qwen-math-results} and~\ref{tab:gsm8k_ablation}) further support the importance of choosing reference directions that match the capabilities one aims to preserve.

\paragraph{Data Efficiency and Update Frequency.}

\begin{table*}[h]
\vspace{-.5cm}
    \centering
    \begin{minipage}[t]{0.48\textwidth}
        \centering
        \caption{Robustness of gradient estimation to sample size budgets using DPO~\citep{rafailov2023direct}.}
        \label{tab:data}
        \resizebox{\linewidth}{!}{
            \begin{tabular}{lccccccc}
                \toprule
                \multirow{2}{*}{\textbf{Model}}
                 & \multicolumn{2}{c}{\textbf{Safety ($\uparrow$)}} & \multicolumn{2}{c}{\textbf{Truthful ($\uparrow$)}} & \multicolumn{2}{c}{\textbf{Helpful ($\uparrow$)}}  \\
                \cmidrule(lr){2-3} \cmidrule(lr){4-5} \cmidrule(lr){6-7}
                 & Stereotype   & StrongReject & SimpleQA &MMLU & IFEval & HHH  & \\    
                \midrule
                Qwen2.5-7B-Instruct Model &96.74      &44.83  &3.33   &73.50  &64.33  &88.77 \\
                \midrule
                w/o ours       &98.47  &63.60  &1.57   &\textbf{73.43} &63.22  &87.52   \\
                + General Data (10k)    &97.32  &46.89  &3.29   &73.21  &62.85  &87.92 \\
                 + w/ 50    &99.43      &70.69  &\textbf{3.37}  &72.57  &61.18  &\textbf{91.53}  \\
                + w/ 100    &\textbf{99.62}     &71.29  &\textbf{3.37}  &72.79  &\textbf{64.33} &90.68     \\
                + w/ 150    &\textbf{99.62}     &72.48  &3.24   &73.00  &61.92  &90.68  \\
                + w/ 200    &99.23      &\textbf{72.64} &3.35   &72.79  &63.40  &90.68  \\
                \bottomrule
            \end{tabular}
        }
    \end{minipage}
    \hfill 
    \begin{minipage}[t]{0.48\textwidth}
        \centering
        \caption{Effect of subspace update frequency on optimization using DPO~\citep{rafailov2023direct} .}
        \label{tab:frequency}
        \resizebox{\linewidth}{!}{
            \begin{tabular}{lccccccc}
                \toprule
                \multirow{2}{*}{\textbf{Model}}
                 & \multicolumn{2}{c}{\textbf{Safety ($\uparrow$)}} & \multicolumn{2}{c}{\textbf{Truthful ($\uparrow$)}} & \multicolumn{2}{c}{\textbf{Helpful ($\uparrow$)}}  \\
                \cmidrule(lr){2-3} \cmidrule(lr){4-5} \cmidrule(lr){6-7}
                 & Stereotype  & StrongReject & SimpleQA &MMLU & IFEval & HHH  & \\    
                \midrule
                Qwen2.5-7B-Instruct Model &96.74      &44.83  &3.33   &73.50  &64.33  &88.77 \\
                \midrule
                w/o our       &98.47  &63.60  &1.57   &73.43  &63.22  &86.28   \\
                + w/ step 2   &\textbf{99.81} &\textbf{80.59} &\textbf{3.49}  &72.36  &59.89  &\textbf{90.68} \\
                + w/ step 5   &99.23  &72.64  &3.35   &72.79  &\textbf{63.40} &\textbf{90.68}\\
                + w/ step 10    &97.51  &65.04  &2.91   &72.93  &61.92  &\textbf{90.68}\\
                + w/ No updating &98.66  &65.07  &1.50   &\textbf{74.93} &60.81  &88.98 \\
                \bottomrule
            \end{tabular}
        }
    \end{minipage}
\vspace{-.2cm}
\end{table*}

\begin{wraptable}{r}{6.8 cm}
\vspace{-.5 cm}
\caption{Training cost comparison.}
\label{tab:training-cost}
\centering
\begin{small}
\resizebox{0.95\linewidth}{!}{
\begin{tabular}{lrr}
\toprule
\textbf{Method} & \textbf{Training Time} & \textbf{Input Tokens (M)} \\
\midrule
SFT             & 1h 49m                & 3.97 \\
+ General Data  & 4h 05m                & 25.21 \\
+ OGPSA (Ours)  & 2h 56m                & 4.39 \\
\bottomrule
\end{tabular}
}
\end{small}
\vspace{-.3 cm}
\end{wraptable}

We evaluate the cost-efficiency and optimization dynamics of OGPSA by varying the reference data size (Table~\ref{tab:data}) and subspace update frequency (Table~\ref{tab:frequency}). First, OGPSA is data-efficient: in this DPO setting, 100--200 samples per reference dimension are sufficient to recover several capability metrics (e.g., \textit{IFEval} and \textit{HHH}) more effectively than the 10K-sample replay baseline. SFT exhibits a similar trend (Appendix Tab.~\ref{tab:data_sft}). This data efficiency does not mean zero overhead: OGPSA increases training time relative to plain SFT because it computes periodic reference gradients, but it remains faster and more token-efficient than general-data mixing in our setup (Table~\ref{tab:training-cost}). Second, dynamically updating the subspace matters. A static projection (``No updating'') can become stale as parameters move, whereas periodic re-estimation---every 5 steps in DPO and 30 steps in SFT in our experiments (Appendix Tab.~\ref{tab:frequency_sft})---provides a better empirical trade-off.

\paragraph{Scalability Analysis}\label{sec:scalability}

\begin{wraptable}{r}{6.2 cm}
\vspace{-.6 cm}
\caption{Scalability analysis of OGPSA across model sizes.
}
\resizebox{0.45\textwidth}{!}{
    \begin{tabular}{lccccccc}
        \toprule
\multirow{2}{*}{\textbf{Model}}
         & \multicolumn{2}{c}{\textbf{Safety ($\uparrow$)}} & \multicolumn{2}{c}{\textbf{Truthful ($\uparrow$)}} & \multicolumn{2}{c}{\textbf{Helpful ($\uparrow$)}}\\
        \cmidrule(lr){2-3} \cmidrule(lr){4-5} \cmidrule(lr){6-7}
        & Stereotype  & StrongReject & SimpleQA &MMLU & IFEval & HHH  & \\    
        \midrule
        \multicolumn{7}{c}{\textbf{Qwen2.5-0.5B-Instruct}} \\
        \midrule
        Instruct Baseline &94.44	&38.43	&0.69	&37.14	&22.55	&65.75	 \\ \cline{2-7}
        SFT              &79.31	&\textbf{86.53}	&0.09	&36.57	&17.19	&63.69	\\
        SFT + Ours       &\textbf{96.74}	&85.38	&\textbf{0.88}	&\textbf{38.00}	&\textbf{23.84}	&\textbf{66.29} \\ \cline{2-7}
        DPO   	         &65.52	&74.75	&0.00	&\textbf{36.78}	&16.08	&66.59\\
        DPO + Ours       &\textbf{99.23}	&\textbf{79.97}	&\textbf{1.06}	&36.57	&\textbf{22.37}	&\textbf{66.69}\\
        \midrule
        \multicolumn{7}{c}{\textbf{Qwen2.5-3B-Instruct}} \\
        \midrule
        Instruct Baseline	&100.00	&43.08	&1.48	&65.36	&54.71	&80.46 \\ \cline{2-7}
        SFT     	        &\textbf{100.00}	&\textbf{69.20}	&0.37	&63.14	&51.39	&\textbf{80.45}\\
        SFT + Ours          &\textbf{100.00}	&66.93	&\textbf{1.13}	&\textbf{63.93}	&\textbf{53.97}	&78.75 \\ \cline{2-7}
        DPO   	            &\textbf{100.00}	&59.78	&0.37	&\textbf{64.64}	&52.49	&\textbf{80.04} \\
        DPO + Ours          &\textbf{100.00}	&\textbf{64.35}	&\textbf{1.90}	&64.50	&\textbf{52.87}	&79.26\\
        \midrule
        \multicolumn{7}{c}{\textbf{Qwen2.5-7B-Instruct}} \\
        \midrule
        Instruct Baseline	&96.74	&44.83	&3.33	&73.50	&64.33	&88.77 \\ \cline{2-7}
        SFT     	        &\textbf{100.00}	&\textbf{90.48}	&0.79	&72.00	&57.30	&\textbf{88.34}\\
        SFT + Ours          &\textbf{100.00}	&87.43	&\textbf{3.61}	&\textbf{73.21}	&\textbf{63.03}	&87.07 \\ \cline{2-7}
        DPO   	            &98.47	&63.60	&1.57	&\textbf{73.43}	&63.22	&87.52\\
        DPO + Ours          &\textbf{99.23}	&\textbf{72.64}	&\textbf{3.35}	&72.79	&\textbf{63.40}	&\textbf{90.68}\\
        \bottomrule
    \end{tabular}
}
\vspace{-0.5cm}
\label{tab:model_size}
\end{wraptable}

To examine whether the trend holds beyond a single model size, we evaluate OGPSA on Qwen2.5 models from 0.5B to 7B parameters. As shown in Table~\ref{tab:model_size}, OGPSA generally improves the retained general-capability metrics under both SFT and DPO while maintaining competitive safety scores. On Qwen2.5-0.5B, for example, it improves \textit{SimpleQA} after SFT from 0.09\% to 0.88\% and improves \textit{Stereotype} from 79.31\% to 96.74\%. Similar patterns appear for the 3B and 7B models, although the magnitude varies by benchmark and model scale. These results support the scalability of the approach within the tested Qwen2.5 family, while leaving larger models and other architectures as future work.

\vspace{-.3cm}
\section{Conclusion and Limitations}
\vspace{-.3cm}

In this work, we framed the alignment tax as a heterogeneous continual learning problem and proposed OGPSA, a lightweight gradient-projection method to mitigate capability regression. By updating along the orthogonal complement of a learned low-rank capability subspace, OGPSA provides a simple, plug-and-play solution for SFT, DPO, and sequential SFT$\rightarrow$DPO pipelines, consistently improving the safety--utility trade-off. While our results demonstrate that explicit control of gradient interference is a highly promising direction, OGPSA's efficacy is bounded by its first-order local approximation, reliance on reference data diversity, and minor computational overhead in distributed settings. To address these boundaries, future work will focus on gradient-conflict diagnostics, stronger compute-matched baselines, comprehensive black-box adaptive safety evaluations, and scaling to larger architectures.

\newpage
\section{Impact Statement}
This work aims to improve the safety--utility trade-off of LLM post-training. A potential positive impact is reducing unnecessary capability degradation when applying safety alignment, which may make aligned systems more useful in benign settings. A potential risk is that improved utility after safety training could be mistaken for comprehensive safety; our method does not guarantee robustness to all jailbreaks, misuse prompts, or deployment distributions. We therefore recommend using OGPSA as one component of a broader safety pipeline that includes red-teaming, policy evaluation, refusal calibration, monitoring, and human oversight.

\section{Large language model assistance}
Large language models were used to polish the manuscript. The authors reviewed and edited the content and take responsibility for the final paper.

\bibliographystyle{plain}
\bibliography{reference}

\newpage
\appendix

\section{Detailed Experimental Setups}
\label{sec:appendix_exp}

In this work, we conduct all our experiments on clusters with 8 NVIDIA A800 GPUs.

\paragraph{Models and Datasets.}
We conduct experiments using two widely adopted instruction-tuned Large Language Models: LLaMA3.1-8B-Instruct~\citep{dubey2024llama} and Qwen2.5-7B-Instruct~\citep{Yang2024Qwen25TR}. 
For the safety alignment phase, we utilize a seed dataset $\mathcal{D}_{\text{safe}}$ consisting of 10k samples sampled from PKU-SafeRLHF~\citep{ji2024pku}, where the SFT labels and DPO chosen labels are refuse response generated by gpt-4omini, and DPO rejected labels are select from the most unsafe answer from the original dataset. To simulate the standard practice of data replay for maintaining general capabilities, we draw 10k pairwise samples from UltraFeedback~\citep{cui2024ultrafeedback} as the general data source. Where the SFT labels and DPO chosen labels are selected from the highest score answer from original dataset, and DPO rejected labels are select from the lowest score answer. For our proposed method (OGPSA), we require a small reference set to estimate the general capability subspace. To efficiently capture the essential dimensions of general utility~\citep{aghajanyan2021intrinsic,zhou2023lima}, we sample a minimal budget of data to represent key competencies:
(1) \textit{Helpfulness}: 200 samples randomly selected from UltraFeedback~\citep{cui2024ultrafeedback}. (2) \textit{Truthfulness}: 200 samples randomly selected from HaluEval~\citep{li2023halueval}, where the SFT labels and DPO chosen labels are correct answers, and DPO rejected labels are hallucinated answers. We preprocess the hallucanated answers to be the same format of the correct answers to prevent reward hacking on the answer format.
This highlights the data efficiency of our approach compared to replay-based methods.

\paragraph{Baselines.}
We evaluate our framework against three standard safety alignment paradigms: Supervised Fine-Tuning (SFT), Direct Preference Optimization (DPO)~\citep{rafailov2023direct}, and Sequential SFT-DPO. 
To benchmark the mitigation of the Alignment Tax, we compare OGPSA against three representative regularization and replay strategies applied to the standard paradigms: (1) \textit{+Merged}: A weight-interpolation method that linearly averages the parameters of the pre-trained model and the safety-aligned model to balance capabilities~\citep{farn2024safeguard,wortsman2022robust}. (2) \textit{+LoRA}: parameter-efficient fine-tuning using low-rank adaptation, which acts as a regularization constraint by updating only a small subset of parameters~\citep{hu2022lora}. (3) \textit{+General Data}: A classic experience replay approach that mixes the 10k general samples from UltraFeedback~\citep{cui2024ultrafeedback} into the safety training data.

\paragraph{Training Details}
\label{sec:appendix_train}
We have done all the training of LLMs with LLaMA-Factory~\citep{zheng2024llamafactory}, which is a popular toolbox for LLM training. Consistent with established protocols~\citep{zhangstair}, all models are trained for $3$ epochs during the SFT stage and $1$ epoch during the DPO stage.
We tune the learning rate $1e-6$ and $\beta$ for DPO from $0.2$. Batch size is fixed as $128$ and weight decay is set to $0$. We adopt a cosine scheduler with a warm-up ratio of $0.1$. Following the official implementation, we set learning rate $1e-4$ for LoRA. For the subspace update frequency $K$. We set $K=30$ for all SFT and $K=5$ for DPO experiments.

\paragraph{Evaluation.}\label{sec:exp_eval}
We employ a comprehensive suite of 10 benchmarks to evaluate the trade-off between safety and general utility. \textit{Safety (Harmlessness):} Following established protocols~\citep{guan2024deliberative}, models are evaluated on their ability to refuse harmful queries. We utilize StrongReject~\citep{souly2024strongreject}, XSTest~\citep{rottger2023xstest}, the toxic split of WildChat~\citep{zhaowildchat}, and the stereotype split of Do-Not-Answer~\citep{wang2023not}. For StrongReject, we report the average defense success score against the top-2 jailbreak attacks (PAIR~\citep{chaojailbreaking} and PAP~\citep{zeng2024johnny}), while reporting refusal rates for other datasets. \textit{General Utility:} We assess diverse capabilities including truthfulness via SimpleQA~\citep{wei2024measuring}, GPQA~\citep{rein2024gpqa}, and MMLU~\citep{hendrycksmeasuring}, and general helpfulness via BIG-bench HHH~\citep{zhou2024beyond} and instruction following via IFEval~\citep{zhou2023instruction}. Additionally, we evaluate adversarial robustness using AdvGLUE~\citep{wang2adversarial}. We report the official metrics for all benchmarks. For evaluation, we use default temperature for generation to guarantee the reproducibility by default. Below, we introduce the benchmarks and corresponding metrics in detail.

For StrongReject~\citep{souly2024strongreject}, we take the official evaluation protocol, which uses GPT-4o-mini to evaluate the responses and gives a rubric-based score reflecting the willingness and capabilities in responding to harmful queries. We follow~\citep{jaech2024openai} and take the goodness score, which is $1-\text{rubric score}$, as the metric. We evaluate models on prompts with no jailbreak in addition to the reported top-2 jailbreak methods PAIR~\citep{chaojailbreaking}, and PAP-Misrepresentation~\citep{zeng2024johnny}. For main results, we only report the average goodness score on the two jailbreak methods, since most methods achieve goodness scores near $1.0$. For XsTest~\citep{rottger2023xstest}, we select the unsafe split to evaluate the resistance to normal harmful queries and follow its official implementation on refusal determination with GPT-4o-mini. We report the sum of full refusal rate and partial refusal rate as the metric. For WildChat~\citep{zhaowildchat}, we filter the conversations with ModerationAPI\footnote{https://platform.openai.com/docs/guides/moderation} and eventually get 219 samples with high toxicity in English. For Stereotype, it is a split for evaluating the model's refusal behavior to queries associated with fairness issues in Do-Not-Answer~\citep{wang2023not}.

\section{Pseudo-code}

\begin{algorithm}[h]
\caption{OGPSA: Orthogonal Gradient Projection for Safety Alignment}
\label{alg:ogpsa}
\begin{algorithmic}[1]
    \renewcommand{\algorithmicrequire}{\textbf{Input:}}
    \renewcommand{\algorithmicensure}{\textbf{Output:}}
    
    \Require Pre-trained parameters $\theta_0$, safety loss $\mathcal{L}_{\mathrm{safe}}$, reference datasets $\{\mathcal{D}^{(i)}_{\mathrm{ref}}\}_{i=1}^M$, refresh period $K$, learning rate $\eta$.
    \Ensure Aligned parameters $\theta_T$.

    \State Initialize $U \leftarrow [\,]$
    \For{$t = 0, \dots, T-1$}
        \If{$t \bmod K = 0$} \Comment{\textcolor{blue}{Dynamic Subspace Construction (refresh every $K$ steps)}}
            \For{$i = 1$ \textbf{to} $M$}
                \State Sample $B^{(i)} \sim \mathcal{D}^{(i)}_{\mathrm{ref}}$
                \State Compute $g^{(i)} \leftarrow \nabla_\theta \mathbb{E}_{\xi \in B^{(i)}}[\ell^{(i)}_{\mathrm{ref}}(\theta_t; \xi)]$
            \EndFor
            \State Construct orthonormal basis $U \leftarrow \mathrm{GramSchmidt}(\{g^{(i)}\}_{i=1}^M)$ \hfill $\triangleright$ $U \equiv U_\tau$
        \EndIf

        \State Compute safety gradient $g_{\mathrm{safe}} \leftarrow \nabla_\theta \mathcal{L}_{\mathrm{safe}}(\theta_t)$  \Comment{\textcolor{blue}{Projected Safety Optimization}} 
        \State Project $\tilde{g}_{\mathrm{safe}} \leftarrow g_{\mathrm{safe}} - U(U^\top g_{\mathrm{safe}})$ \hfill $\triangleright$ \Comment{\textcolor{blue}{Remove conflicting components}} 
        \State Update $\theta_{t+1} \leftarrow \theta_t - \eta \tilde{g}_{\mathrm{safe}}$
    \EndFor
    \State \Return $\theta_T$
\end{algorithmic}
\end{algorithm}

\section{Theoretical Derivations}
\label{sec:appendix_proof}

In this section, we provide the general mathematical foundation for Proposition 1 presented in the main text. We prove that for any differentiable function, the direction of steepest descent restricted to a linear subspace is equivalent to the negative gradient projected onto that subspace.

\subsection{Steepest Descent Direction in a Linear Subspace}

\paragraph{Formalization.}
Let $V = \mathbb{R}^d$ be a $d$-dimensional Euclidean space (representing the parameter space of the LLM) equipped with the standard inner product $\langle \cdot, \cdot \rangle$ and the induced norm $\| \cdot \|$.
Consider the following definitions:
\begin{itemize}
    \item Let $f: \mathbb{R}^d \to \mathbb{R}$ be a differentiable scalar function (representing the loss function $\mathcal{L}$).
    \item Let $g = \nabla f(\theta) \in \mathbb{R}^d$ denote the gradient of $f$ at point $\theta$.
    \item Let $\mathcal{S} \subseteq \mathbb{R}^d$ be a linear subspace of $V$ (representing the allowable optimization subspace, e.g., $\mathcal{S}_{\text{gen}}^{\perp}$).
    \item Let $P_{\mathcal{S}}: \mathbb{R}^d \to \mathcal{S}$ denote the orthogonal projection operator onto $\mathcal{S}$.
\end{itemize}

\textbf{Objective:} We seek a unit vector $v \in \mathcal{S}$ (i.e., $\|v\| = 1$) that maximizes the rate of descent, equivalent to minimizing the directional derivative $D_v f(\theta)$.

\begin{lemma}[Optimal Descent in Subspace]
    Assume $P_{\mathcal{S}}(g) \neq 0$. The direction $v^* \in \mathcal{S}$ that minimizes the directional derivative of $f$ is given by:
    \begin{equation}
        v^* = - \frac{P_{\mathcal{S}}(g)}{\|P_{\mathcal{S}}(g)\|}.
    \end{equation}
    In other words, the steepest descent direction within a subspace is the negative of the orthogonally projected gradient.
\end{lemma}

\begin{proof}
    \textbf{Step 1: Definition of the Directional Derivative.}
    The directional derivative of $f$ at $\theta$ along $v$ is given by:
    \begin{equation}
        D_v f(\theta) = \langle \nabla f(\theta), v \rangle = \langle g, v \rangle.
    \end{equation}

    \textbf{Step 2: Orthogonal Decomposition.}
    By the Projection Theorem, the gradient $g$ can be uniquely decomposed into a component within $\mathcal{S}$ and a component orthogonal to $\mathcal{S}$:
    \begin{equation}
        g = g_{\mathcal{S}} + g_{\perp},
    \end{equation}
    where $g_{\mathcal{S}} = P_{\mathcal{S}}(g) \in \mathcal{S}$ and $g_{\perp} \in \mathcal{S}^{\perp}$. By definition, for any vector $u \in \mathcal{S}$, the inner product $\langle g_{\perp}, u \rangle = 0$.

    \textbf{Step 3: Simplifying the Objective.}
    We minimize $\langle g, v \rangle$ subject to $v \in \mathcal{S}$ and $\|v\| = 1$. Substituting the decomposition:
    \begin{equation}
        \langle g, v \rangle = \langle g_{\mathcal{S}} + g_{\perp}, v \rangle = \langle g_{\mathcal{S}}, v \rangle + \underbrace{\langle g_{\perp}, v \rangle}_{0} = \langle g_{\mathcal{S}}, v \rangle.
    \end{equation}

    \textbf{Step 4: Minimization via Cauchy-Schwarz.}
    The problem reduces to minimizing the inner product $\langle g_{\mathcal{S}}, v \rangle$ subject to unit norm. By the Cauchy-Schwarz~\citep{bjorck1994numerics,leon2013gram} inequality:
    \begin{equation}
        |\langle g_{\mathcal{S}}, v \rangle| \leq \|g_{\mathcal{S}}\| \|v\| = \|g_{\mathcal{S}}\|.
    \end{equation}
    This implies:
    \begin{equation}
        -\|g_{\mathcal{S}}\| \leq \langle g_{\mathcal{S}}, v \rangle \leq \|g_{\mathcal{S}}\|.
    \end{equation}
    The lower bound (maximum descent) is achieved if and only if $v$ is collinear to $g_{\mathcal{S}}$ and points in the opposite direction. Thus, the optimal vector is:
    \begin{equation}
        v^* = - \frac{g_{\mathcal{S}}}{\|g_{\mathcal{S}}\|} = - \frac{P_{\mathcal{S}}(g)}{\|P_{\mathcal{S}}(g)\|}.
    \end{equation}
\end{proof}

\textbf{Connection to Main Text:} In the context of OGPSA, the subspace $\mathcal{S}$ corresponds to the null space of general capabilities $\mathcal{S}_{\text{gen}}^{\perp}$, and the function $f$ corresponds to the safety loss $\mathcal{L}_{\text{safe}}$. This lemma proves that our update rule follows the optimal path for safety optimization constrained within the non-forgetting zone.

\newpage

\section{Appendix Results}\label{sec:appendix_results}
\subsection{Overall Performance of Lllama}

\begin{figure}[h] 
    \centering
    \includegraphics[width=0.5\linewidth]{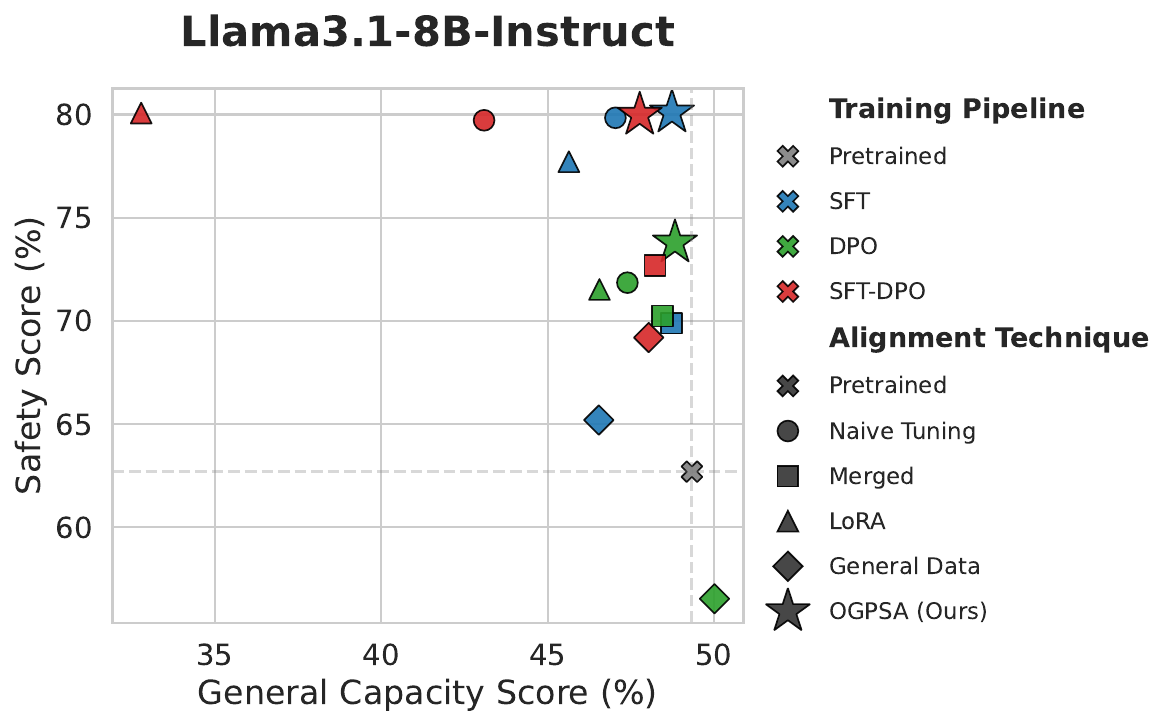}
    \caption{\textbf{Overall performance of alignment strategies on Llama3.1-8B-Instruct.} 
    We report the aggregate \textit{Safety Score} (avg. of 4 datasets) and \textit{General Capacity Score} (avg. of 6 datasets); see Table~\ref{tab:main} for details.
    }
    \label{fig:llama_results}
    \vspace{-.7cm}
\end{figure}

\subsection{Impact of Subspace Dimensionality and Diversity using SFT}

\begin{table*}[h]
\centering
\caption{\textbf{Effect of general capability subspace composition on alignment outcomes using SFT on Qwen.} We investigate how the diversity of reference data (Helpfulness vs. Truthfulness) and the dimensionality of the constraint subspace (1 vs. 2 directions) impact the alignment outcomes. 
The \textbf{best} results are marked in \textbf{bold}.
}
\resizebox{0.8\textwidth}{!}{
    \begin{tabular}{lcccccc}
        \toprule
\multirow{2}{*}{\textbf{Model}}
         & \multicolumn{2}{c}{\textbf{Safety ($\uparrow$)}} & \multicolumn{2}{c}{\textbf{Truthful ($\uparrow)$}} & \multicolumn{2}{c}{\textbf{Helpful ($\uparrow$)}} \\
        \cmidrule(lr){2-3} \cmidrule(lr){4-5} \cmidrule(lr){6-7}
        & Stereotype  & StrongReject & SimpleQA &MMLU & IFEval & HHH \\   
        \midrule
        Qwen2.5-7B-Instruct Model &96.74		&44.83	&3.33	&73.50	&64.33	&88.77	\\
        \midrule
        w/o ours    &\textbf{100.00}	&90.48	&0.79	&72.00	&57.30	&88.34 	\\
        + 1 dim (Helpfulness)    &\textbf{100.00}	&\textbf{91.60}	&1.16	&\textbf{73.36}	&58.04 &87.50 \\
        + 1 dim (Truthfulness)         &\textbf{100.00}	&86.09	&\textbf{3.70}	&73.14	&62.85 &87.92 \\
        + 1 dim (Mixed) &99.81	&88.60	&3.49	&72.50	&\textbf{63.40} &\textbf{88.74} \\
        + 2 dim (Helpfulness$+$Truthfulness)   &\textbf{100.00}	&87.43	&3.61	&73.21	&63.03 &87.07 \\
        \bottomrule
    \end{tabular}
}
\vspace{-0.3cm}
\label{tab:base_direction_sft}
\end{table*}

\subsection{Generalization to the Mathematical Domain}

\begin{table}[h]
\caption{Performance comparison on mathematical benchmarks for Qwen2.5-7B-Instruct. Bold indicates the best performance within each alignment category.}
\label{tab:qwen-math-results}
\centering
\begin{small}
\begin{tabular}{lccc}
\toprule
\textbf{Method} & \textbf{GSM8K} & \textbf{AIME2024} & \textbf{AIME2025} \\
\midrule
SFT                & 89.61 & 8.82  & 6.08 \\
+ OGPSA (Ours)     & \textbf{90.83} & \textbf{9.22} & \textbf{6.47} \\
\midrule
DPO                & 90.74 & 11.96 & \textbf{9.41} \\
+ OGPSA (Ours)     & \textbf{91.21} & \textbf{12.35} & 8.82 \\
\bottomrule
\end{tabular}
\end{small}
\end{table}

\begin{table*}[h]
\centering
\caption{\textbf{Robustness to Swapping the Reference Set.} We investigate the alignment outcomes when entirely replacing the reference sets with 200 samples from GSM8K (a math-only dataset) during the SFT phase. The \textbf{best} results are marked in \textbf{bold}.}
\resizebox{\textwidth}{!}{
    \begin{tabular}{lcccccccccccc}
        \toprule
        \multirow{2}{*}{\textbf{Method}} & \multicolumn{12}{c}{\textbf{Benchmarks}} \\
        \cmidrule(lr){2-13}
         & XSTest & WildChat & Stereotype & StrongReject & SimpleQA & GPQA & MMLU & IFEval & HHH & AdvGLUE & AIME2024 & AIME2025 \\   
        \midrule
        SFT & \textbf{87.00} & 64.20 & \textbf{100.00} & 90.48 & 0.79 & \textbf{33.00} & 72.00 & 57.30 & 88.34 & 69.48 & 8.82 & 6.08 \\
        + OGPSA (w/ GSM8K) & 86.00 & \textbf{66.20} & \textbf{100.00} & \textbf{91.23} & \textbf{3.44} & 32.66 & \textbf{73.07} & \textbf{60.63} & \textbf{90.38} & \textbf{69.58} & \textbf{10.98} & \textbf{7.65} \\
        \bottomrule
    \end{tabular}
}
\vspace{-0.3cm}
\label{tab:gsm8k_ablation}
\end{table*}

\subsection{Impact of Sample Size Budgets for Gradient Estimation}

\begin{table*}[h]
\centering
\caption{\textbf{Robustness of gradient estimation to sample size budgets using DPO~\citep{rafailov2023direct} on Qwen.} We evaluate the performance stability as the number of samples used to estimate the reference capability gradients increases. Our method remains effective even with limited data budgets. The \textbf{best} results are marked in \textbf{bold}.
}
\resizebox{0.8\textwidth}{!}{
    \begin{tabular}{lccccccc}
        \toprule
\multirow{2}{*}{\textbf{Model}}
         & \multicolumn{2}{c}{\textbf{Safety ($\uparrow$)}} & \multicolumn{2}{c}{\textbf{Truthful ($\uparrow)$}} & \multicolumn{2}{c}{\textbf{Helpful ($\uparrow$)}}  \\
        \cmidrule(lr){2-3} \cmidrule(lr){4-5} \cmidrule(lr){6-7}
        & Stereotype   & StrongReject & SimpleQA &MMLU & IFEval & HHH  & \\    
        \midrule
        Qwen2.5-7B-Instruct Model &96.74		&44.83	&3.33	&73.50	&64.33	&88.77 \\
        \midrule
        w/o ours     &\textbf{100.00}	&\textbf{90.48}	&0.79	&72.00	&57.30	&88.34	\\
        + General Data (10k)       &99.04		&70.38	&3.47	&69.93	&53.42	&86.28 \\
        + w/ 50   &\textbf{100.00}		&86.87	&3.35	&71.71	&61.74	&89.16	  \\
        + w/ 100  &\textbf{100.00}		&85.61	&3.54	&72.43	&63.22	&88.73	     \\
        + w/ 150  &\textbf{100.00}		&87.37	&3.47	&\textbf{73.21}	&\textbf{63.03}	&\textbf{89.97}	  \\
        + w/ 200  &\textbf{100.00}		&87.43	&\textbf{3.61}	&\textbf{73.21}	&\textbf{63.03}	&87.07	 \\
        \midrule
    \end{tabular}
}
\vspace{-0.3cm}
\label{tab:data_sft}
\end{table*}

\subsection{Impact of Subspace Update Frequency using SFT}

\begin{table}[h]
\centering
\caption{Effect of subspace update frequency on optimization dynamics using SFT on Qwen. We compare static subspaces against dynamic updates at varying intervals. 
The \textbf{best} results are marked in \textbf{bold}.
}
\resizebox{0.8\textwidth}{!}{
    \begin{tabular}{lccccccc}
        \toprule
\multirow{2}{*}{\textbf{Model}}
         & \multicolumn{2}{c}{\textbf{Safety ($\uparrow$)}} & \multicolumn{2}{c}{\textbf{Truthful ($\uparrow)$}} & \multicolumn{2}{c}{\textbf{Helpful ($\uparrow$)}}  \\
        \cmidrule(lr){2-3} \cmidrule(lr){4-5} \cmidrule(lr){6-7}
        & Stereotype  & StrongReject & SimpleQA &MMLU & IFEval & HHH  & \\    
        \midrule
        Qwen2.5-7B-Instruct Model &96.74		&44.83	&3.33	&73.50	&64.33	&88.77 \\
        \midrule
        w/o ours    &\textbf{100.00}	&\textbf{90.48}	&0.79	&72.00	&57.30	&88.34	\\
        + w/ step 5   &\textbf{100.00}	&89.76	&\textbf{3.68}	&71.36	&61.92	&88.27  \\
        + w/ step 15   &\textbf{100.00}	&87.69	&3.56	&72.29	&\textbf{64.14}	&88.34 \\
        + w/ step 30   &\textbf{100.00}	&87.43	&3.61	&73.21	&63.03	&87.07 \\
        + w/ No updating &\textbf{100.00}	&89.04	&1.02	&\textbf{73.29}	&58.04	&\textbf{89.57} \\
        \bottomrule
    \end{tabular}
}
\vspace{-0.3cm}
\label{tab:frequency_sft}
\end{table}

\clearpage
\newpage
\section*{NeurIPS Paper Checklist}

\begin{enumerate}

\item {\bf Claims}
    \item[] Question: Do the main claims made in the abstract and introduction accurately reflect the paper's contributions and scope?
    \item[] Answer: \answerYes{} 
    \item[] Justification: The main claims made in the abstract and introduction accurately reflect the paper's contributions and scope. 
    \item[] Guidelines:
    \begin{itemize}
        \item The answer \answerNA{} means that the abstract and introduction do not include the claims made in the paper.
        \item The abstract and/or introduction should clearly state the claims made, including the contributions made in the paper and important assumptions and limitations. A \answerNo{} or \answerNA{} answer to this question will not be perceived well by the reviewers. 
        \item The claims made should match theoretical and experimental results, and reflect how much the results can be expected to generalize to other settings. 
        \item It is fine to include aspirational goals as motivation as long as it is clear that these goals are not attained by the paper. 
    \end{itemize}

\item {\bf Limitations}
    \item[] Question: Does the paper discuss the limitations of the work performed by the authors?
    \item[] Answer: \answerYes{} 
    \item[] Justification: the paper discuss the limitations in discussion section. 
    \item[] Guidelines:
    \begin{itemize}
        \item The answer \answerNA{} means that the paper has no limitation while the answer \answerNo{} means that the paper has limitations, but those are not discussed in the paper. 
        \item The authors are encouraged to create a separate ``Limitations'' section in their paper.
        \item The paper should point out any strong assumptions and how robust the results are to violations of these assumptions (e.g., independence assumptions, noiseless settings, model well-specification, asymptotic approximations only holding locally). The authors should reflect on how these assumptions might be violated in practice and what the implications would be.
        \item The authors should reflect on the scope of the claims made, e.g., if the approach was only tested on a few datasets or with a few runs. In general, empirical results often depend on implicit assumptions, which should be articulated.
        \item The authors should reflect on the factors that influence the performance of the approach. For example, a facial recognition algorithm may perform poorly when image resolution is low or images are taken in low lighting. Or a speech-to-text system might not be used reliably to provide closed captions for online lectures because it fails to handle technical jargon.
        \item The authors should discuss the computational efficiency of the proposed algorithms and how they scale with dataset size.
        \item If applicable, the authors should discuss possible limitations of their approach to address problems of privacy and fairness.
        \item While the authors might fear that complete honesty about limitations might be used by reviewers as grounds for rejection, a worse outcome might be that reviewers discover limitations that aren't acknowledged in the paper. The authors should use their best judgment and recognize that individual actions in favor of transparency play an important role in developing norms that preserve the integrity of the community. Reviewers will be specifically instructed to not penalize honesty concerning limitations.
    \end{itemize}

\item {\bf Theory assumptions and proofs}
    \item[] Question: For each theoretical result, does the paper provide the full set of assumptions and a complete (and correct) proof?
    \item[] Answer: \answerYes{} 
    \item[] Justification: the paper provide the full set of assumptions and a complete (and correct) proof for each theoretical result. 
    \item[] Guidelines:
    \begin{itemize}
        \item The answer \answerNA{} means that the paper does not include theoretical results. 
        \item All the theorems, formulas, and proofs in the paper should be numbered and cross-referenced.
        \item All assumptions should be clearly stated or referenced in the statement of any theorems.
        \item The proofs can either appear in the main paper or the supplemental material, but if they appear in the supplemental material, the authors are encouraged to provide a short proof sketch to provide intuition. 
        \item Inversely, any informal proof provided in the core of the paper should be complemented by formal proofs provided in appendix or supplemental material.
        \item Theorems and Lemmas that the proof relies upon should be properly referenced. 
    \end{itemize}

    \item {\bf Experimental result reproducibility}
    \item[] Question: Does the paper fully disclose all the information needed to reproduce the main experimental results of the paper to the extent that it affects the main claims and/or conclusions of the paper (regardless of whether the code and data are provided or not)?
    \item[] Answer: \answerYes{} 
    \item[] Justification: The paper fully disclose all the information needed to reproduce the main experimental results of the paper to the extent that it affects the main claims and/or conclusions of the paper. 
    \item[] Guidelines:
    \begin{itemize}
        \item The answer \answerNA{} means that the paper does not include experiments.
        \item If the paper includes experiments, a \answerNo{} answer to this question will not be perceived well by the reviewers: Making the paper reproducible is important, regardless of whether the code and data are provided or not.
        \item If the contribution is a dataset and\slash or model, the authors should describe the steps taken to make their results reproducible or verifiable. 
        \item Depending on the contribution, reproducibility can be accomplished in various ways. For example, if the contribution is a novel architecture, describing the architecture fully might suffice, or if the contribution is a specific model and empirical evaluation, it may be necessary to either make it possible for others to replicate the model with the same dataset, or provide access to the model. In general. releasing code and data is often one good way to accomplish this, but reproducibility can also be provided via detailed instructions for how to replicate the results, access to a hosted model (e.g., in the case of a large language model), releasing of a model checkpoint, or other means that are appropriate to the research performed.
        \item While NeurIPS does not require releasing code, the conference does require all submissions to provide some reasonable avenue for reproducibility, which may depend on the nature of the contribution. For example
        \begin{enumerate}
            \item If the contribution is primarily a new algorithm, the paper should make it clear how to reproduce that algorithm.
            \item If the contribution is primarily a new model architecture, the paper should describe the architecture clearly and fully.
            \item If the contribution is a new model (e.g., a large language model), then there should either be a way to access this model for reproducing the results or a way to reproduce the model (e.g., with an open-source dataset or instructions for how to construct the dataset).
            \item We recognize that reproducibility may be tricky in some cases, in which case authors are welcome to describe the particular way they provide for reproducibility. In the case of closed-source models, it may be that access to the model is limited in some way (e.g., to registered users), but it should be possible for other researchers to have some path to reproducing or verifying the results.
        \end{enumerate}
    \end{itemize}

\item {\bf Open access to data and code}
    \item[] Question: Does the paper provide open access to the data and code, with sufficient instructions to faithfully reproduce the main experimental results, as described in supplemental material?
    \item[] Answer: \answerYes{} 
    \item[] Justification: The paper provide open access to the data and code, with sufficient instructions to faithfully reproduce the main experimental results, as described in supplemental material. 
    \item[] Guidelines:
    \begin{itemize}
        \item The answer \answerNA{} means that paper does not include experiments requiring code.
        \item Please see the NeurIPS code and data submission guidelines (\url{https://neurips.cc/public/guides/CodeSubmissionPolicy}) for more details.
        \item While we encourage the release of code and data, we understand that this might not be possible, so \answerNo{} is an acceptable answer. Papers cannot be rejected simply for not including code, unless this is central to the contribution (e.g., for a new open-source benchmark).
        \item The instructions should contain the exact command and environment needed to run to reproduce the results. See the NeurIPS code and data submission guidelines (\url{https://neurips.cc/public/guides/CodeSubmissionPolicy}) for more details.
        \item The authors should provide instructions on data access and preparation, including how to access the raw data, preprocessed data, intermediate data, and generated data, etc.
        \item The authors should provide scripts to reproduce all experimental results for the new proposed method and baselines. If only a subset of experiments are reproducible, they should state which ones are omitted from the script and why.
        \item At submission time, to preserve anonymity, the authors should release anonymized versions (if applicable).
        \item Providing as much information as possible in supplemental material (appended to the paper) is recommended, but including URLs to data and code is permitted.
    \end{itemize}

\item {\bf Experimental setting/details}
    \item[] Question: Does the paper specify all the training and test details (e.g., data splits, hyperparameters, how they were chosen, type of optimizer) necessary to understand the results?
    \item[] Answer: \answerYes{} 
    \item[] Justification: The paper specify all the training and test details (e.g., data splits, hyperparameters, how they were chosen, type of optimizer) necessary to understand the results. 
    \item[] Guidelines:
    \begin{itemize}
        \item The answer \answerNA{} means that the paper does not include experiments.
        \item The experimental setting should be presented in the core of the paper to a level of detail that is necessary to appreciate the results and make sense of them.
        \item The full details can be provided either with the code, in appendix, or as supplemental material.
    \end{itemize}

\item {\bf Experiment statistical significance}
    \item[] Question: Does the paper report error bars suitably and correctly defined or other appropriate information about the statistical significance of the experiments?
    \item[] Answer: \answerNo{} 
    \item[] Justification: The paper does not report error bars of results. 
    \item[] Guidelines:
    \begin{itemize}
        \item The answer \answerNA{} means that the paper does not include experiments.
        \item The authors should answer \answerYes{} if the results are accompanied by error bars, confidence intervals, or statistical significance tests, at least for the experiments that support the main claims of the paper.
        \item The factors of variability that the error bars are capturing should be clearly stated (for example, train/test split, initialization, random drawing of some parameter, or overall run with given experimental conditions).
        \item The method for calculating the error bars should be explained (closed form formula, call to a library function, bootstrap, etc.)
        \item The assumptions made should be given (e.g., Normally distributed errors).
        \item It should be clear whether the error bar is the standard deviation or the standard error of the mean.
        \item It is OK to report 1-sigma error bars, but one should state it. The authors should preferably report a 2-sigma error bar than state that they have a 96\% CI, if the hypothesis of Normality of errors is not verified.
        \item For asymmetric distributions, the authors should be careful not to show in tables or figures symmetric error bars that would yield results that are out of range (e.g., negative error rates).
        \item If error bars are reported in tables or plots, the authors should explain in the text how they were calculated and reference the corresponding figures or tables in the text.
    \end{itemize}

\item {\bf Experiments compute resources}
    \item[] Question: For each experiment, does the paper provide sufficient information on the computer resources (type of compute workers, memory, time of execution) needed to reproduce the experiments?
    \item[] Answer: \answerYes{} 
    \item[] Justification: The paper provide sufficient information on the computer resources (type of compute workers, memory, time of execution) needed to reproduce the experiments. 
    \item[] Guidelines:
    \begin{itemize}
        \item The answer \answerNA{} means that the paper does not include experiments.
        \item The paper should indicate the type of compute workers CPU or GPU, internal cluster, or cloud provider, including relevant memory and storage.
        \item The paper should provide the amount of compute required for each of the individual experimental runs as well as estimate the total compute. 
        \item The paper should disclose whether the full research project required more compute than the experiments reported in the paper (e.g., preliminary or failed experiments that didn't make it into the paper). 
    \end{itemize}
    
\item {\bf Code of ethics}
    \item[] Question: Does the research conducted in the paper conform, in every respect, with the NeurIPS Code of Ethics \url{https://neurips.cc/public/EthicsGuidelines}?
    \item[] Answer: \answerYes{} 
    \item[] Justification: The research conducted in the paper conform, in every respect, with the NeurIPS Code of Ethics \url{https://neurips.cc/public/EthicsGuidelines}. 
    \item[] Guidelines:
    \begin{itemize}
        \item The answer \answerNA{} means that the authors have not reviewed the NeurIPS Code of Ethics.
        \item If the authors answer \answerNo, they should explain the special circumstances that require a deviation from the Code of Ethics.
        \item The authors should make sure to preserve anonymity (e.g., if there is a special consideration due to laws or regulations in their jurisdiction).
    \end{itemize}

\item {\bf Broader impacts}
    \item[] Question: Does the paper discuss both potential positive societal impacts and negative societal impacts of the work performed?
    \item[] Answer: \answerYes{} 
    \item[] Justification: The paper discuss both potential positive societal impacts and negative societal impacts of the work performed. 
    \item[] Guidelines:
    \begin{itemize}
        \item The answer \answerNA{} means that there is no societal impact of the work performed.
        \item If the authors answer \answerNA{} or \answerNo, they should explain why their work has no societal impact or why the paper does not address societal impact.
        \item Examples of negative societal impacts include potential malicious or unintended uses (e.g., disinformation, generating fake profiles, surveillance), fairness considerations (e.g., deployment of technologies that could make decisions that unfairly impact specific groups), privacy considerations, and security considerations.
        \item The conference expects that many papers will be foundational research and not tied to particular applications, let alone deployments. However, if there is a direct path to any negative applications, the authors should point it out. For example, it is legitimate to point out that an improvement in the quality of generative models could be used to generate Deepfakes for disinformation. On the other hand, it is not needed to point out that a generic algorithm for optimizing neural networks could enable people to train models that generate Deepfakes faster.
        \item The authors should consider possible harms that could arise when the technology is being used as intended and functioning correctly, harms that could arise when the technology is being used as intended but gives incorrect results, and harms following from (intentional or unintentional) misuse of the technology.
        \item If there are negative societal impacts, the authors could also discuss possible mitigation strategies (e.g., gated release of models, providing defenses in addition to attacks, mechanisms for monitoring misuse, mechanisms to monitor how a system learns from feedback over time, improving the efficiency and accessibility of ML).
    \end{itemize}
    
\item {\bf Safeguards}
    \item[] Question: Does the paper describe safeguards that have been put in place for responsible release of data or models that have a high risk for misuse (e.g., pre-trained language models, image generators, or scraped datasets)?
    \item[] Answer: \answerNA{} 
    \item[] Justification: The paper poses no such risks. 
    \item[] Guidelines:
    \begin{itemize}
        \item The answer \answerNA{} means that the paper poses no such risks.
        \item Released models that have a high risk for misuse or dual-use should be released with necessary safeguards to allow for controlled use of the model, for example by requiring that users adhere to usage guidelines or restrictions to access the model or implementing safety filters. 
        \item Datasets that have been scraped from the Internet could pose safety risks. The authors should describe how they avoided releasing unsafe images.
        \item We recognize that providing effective safeguards is challenging, and many papers do not require this, but we encourage authors to take this into account and make a best faith effort.
    \end{itemize}

\item {\bf Licenses for existing assets}
    \item[] Question: Are the creators or original owners of assets (e.g., code, data, models), used in the paper, properly credited and are the license and terms of use explicitly mentioned and properly respected?
    \item[] Answer: \answerYes{} 
    \item[] Justification: The creators or original owners of assets (e.g., code, data, models), used in the paper, properly credited and are the license and terms of use explicitly mentioned and properly respected. 
    \item[] Guidelines:
    \begin{itemize}
        \item The answer \answerNA{} means that the paper does not use existing assets.
        \item The authors should cite the original paper that produced the code package or dataset.
        \item The authors should state which version of the asset is used and, if possible, include a URL.
        \item The name of the license (e.g., CC-BY 4.0) should be included for each asset.
        \item For scraped data from a particular source (e.g., website), the copyright and terms of service of that source should be provided.
        \item If assets are released, the license, copyright information, and terms of use in the package should be provided. For popular datasets, \url{paperswithcode.com/datasets} has curated licenses for some datasets. Their licensing guide can help determine the license of a dataset.
        \item For existing datasets that are re-packaged, both the original license and the license of the derived asset (if it has changed) should be provided.
        \item If this information is not available online, the authors are encouraged to reach out to the asset's creators.
    \end{itemize}

\item {\bf New assets}
    \item[] Question: Are new assets introduced in the paper well documented and is the documentation provided alongside the assets?
    \item[] Answer: \answerNA{} 
    \item[] Justification: The paper does not release new assets. 
    \item[] Guidelines:
    \begin{itemize}
        \item The answer \answerNA{} means that the paper does not release new assets.
        \item Researchers should communicate the details of the dataset\slash code\slash model as part of their submissions via structured templates. This includes details about training, license, limitations, etc. 
        \item The paper should discuss whether and how consent was obtained from people whose asset is used.
        \item At submission time, remember to anonymize your assets (if applicable). You can either create an anonymized URL or include an anonymized zip file.
    \end{itemize}

\item {\bf Crowdsourcing and research with human subjects}
    \item[] Question: For crowdsourcing experiments and research with human subjects, does the paper include the full text of instructions given to participants and screenshots, if applicable, as well as details about compensation (if any)? 
    \item[] Answer: \answerNA{} 
    \item[] Justification: The paper does not involve crowdsourcing nor research with human subjects. 
    \item[] Guidelines:
    \begin{itemize}
        \item The answer \answerNA{} means that the paper does not involve crowdsourcing nor research with human subjects.
        \item Including this information in the supplemental material is fine, but if the main contribution of the paper involves human subjects, then as much detail as possible should be included in the main paper. 
        \item According to the NeurIPS Code of Ethics, workers involved in data collection, curation, or other labor should be paid at least the minimum wage in the country of the data collector. 
    \end{itemize}

\item {\bf Institutional review board (IRB) approvals or equivalent for research with human subjects}
    \item[] Question: Does the paper describe potential risks incurred by study participants, whether such risks were disclosed to the subjects, and whether Institutional Review Board (IRB) approvals (or an equivalent approval/review based on the requirements of your country or institution) were obtained?
    \item[] Answer: \answerNA{} 
    \item[] Justification: the paper does not involve crowdsourcing nor research with human subjects. 
    \item[] Guidelines:
    \begin{itemize}
        \item The answer \answerNA{} means that the paper does not involve crowdsourcing nor research with human subjects.
        \item Depending on the country in which research is conducted, IRB approval (or equivalent) may be required for any human subjects research. If you obtained IRB approval, you should clearly state this in the paper. 
        \item We recognize that the procedures for this may vary significantly between institutions and locations, and we expect authors to adhere to the NeurIPS Code of Ethics and the guidelines for their institution. 
        \item For initial submissions, do not include any information that would break anonymity (if applicable), such as the institution conducting the review.
    \end{itemize}

\item {\bf Declaration of LLM usage}
    \item[] Question: Does the paper describe the usage of LLMs if it is an important, original, or non-standard component of the core methods in this research? Note that if the LLM is used only for writing, editing, or formatting purposes and does \emph{not} impact the core methodology, scientific rigor, or originality of the research, declaration is not required.
    \item[] Answer: \answerYes{} 
    \item[] Justification: Large language models were used to polish the manuscript. The authors have thoroughly reviewed and edited all content and take full responsibility for the published work. 
    \item[] Guidelines:
    \begin{itemize}
        \item The answer \answerNA{} means that the core method development in this research does not involve LLMs as any important, original, or non-standard components.
        \item Please refer to our LLM policy in the NeurIPS handbook for what should or should not be described.
    \end{itemize}

\end{enumerate}

\end{document}